\def\year{2020}\relax
\documentclass[letterpaper]{article} 
\usepackage{aaai20}  
\usepackage{times}  
\usepackage{helvet} 
\usepackage{courier}  
\usepackage[hyphens]{url}  
\usepackage{graphicx} 
\usepackage{graphicx}  
\usepackage{amsmath}
\usepackage{amsthm}
\usepackage{amssymb}
\usepackage{amsfonts}

\usepackage{booktabs}
\usepackage[usenames,dvipsnames]{color}
\usepackage{colortbl}

\definecolor{mygray}{gray}{.9}
\usepackage{diagbox}

\usepackage{dsfont}
\usepackage{graphicx}
\usepackage{graphics}
\usepackage{subfigure}
\newtheorem{theorem}{\textbf{Theorem}}

\newtheorem{definition}{Definition}
\newtheorem{lemma}{Lemma}
\newtheorem{corollary}{Corollary}

\usepackage{algorithm}
\usepackage{algorithmic}

\usepackage{multirow}
\usepackage{lipsum}

\newcommand\dif{\mathop{}\!\mathrm{d}}
\title{Subjectivity Learning Theory towards Artificial General Intelligence}

\author{\Large \textbf{Xin Su, Shangqi Guo, Feng Chen}\\ 
	\textsuperscript{\rm 1}Department of Automation, Tsinghua University\\ 
	suxin16, gsq15@mails.tsinghua.edu.cn; chenfeng@mail.tsinghua.edu.cn 
}

\begin{document}

\maketitle

\begin{abstract}
The construction of artificial general intelligence (AGI) was a long-term goal of AI research aiming to deal with the complex data in the real world and make reasonable judgments in various cases like a human. However, the current AI creations, referred to as ``Narrow AI", are limited to a specific problem. The constraints come from two basic assumptions of data, which are independent and identical distributed samples and single-valued mapping between inputs and outputs. We completely break these constraints and develop the subjectivity learning theory for general intelligence. We assign the mathematical meaning for the philosophical concept of subjectivity and build the data representation of general intelligence. Under the subjectivity representation, then the global risk is constructed as the new learning goal. We prove that subjectivity learning holds a lower risk bound than traditional machine learning. Moreover, we propose the principle of empirical global risk minimization (EGRM) as the subjectivity learning process in practice, establish the condition of consistency, and present triple variables for controlling the total risk bound. The subjectivity learning is a novel learning theory for unconstrained real data and provides a path to develop AGI.


\end{abstract}

\section{Introduction}

In the past few decades, artificial intelligence (AI) research has reached or even exceeded human-level performance on many specific problems \cite{silver2016mastering,mnih2015human,he2015delving}. In the current implementation of AI, the learning relies heavily on datasets, which are built by artificially distinguished and collected data samples of specific domains. The learning machine obtains abilities by minimizing the risk of specific problems, as shown in Figure \ref{pic 1 p2}. This method of AI introduces a basic data assumption \cite{Vapnik2003Statistical} that \emph{ data samples are independent and identically distributed (i.i.d)}. This assumption conforms to the characteristics of datasets and makes AI a solvable problem. However, it also limits the current AI machines to a specific ``intelligent" behaviors in a determined environment, which is referred to as ``Narrow AI"\cite{Kurzweil2005THE}. When facing a general learning case, these AI creations still have many problems, such as task specificity, weak generalization, catastrophic forgetting. Some recent works attempt to solve these problems. For instance, multi-task learning tries to learn multiple data distribution at the same time \cite{Evgeniou2004Regularized,Kendall2018Multi}; continual learning tries to learn sequential and unstable distributed data \cite{Zenke2017Continual,Aljundi2019Online,farquhar2019unifying}; transfer learning tries to extend the ability of one data distribution to another related one  \cite{Santoro2016One,Ren2018Meta}. Although the above works try to break the limits of i.i.d samples to achieve greater progress, they still cannot deal with the general learning scenario. The essential reason is another basic data assumption in traditional machine learning: \emph{The single-valued mapping function of inputs to outputs}. The current AI studies are learning the mapping function $y=f(x)$ or $F(y|x)$ from input to output for all data.

For the problem of general intelligence, not only the i.i.d assumption but the single-valued mapping assumption in traditional learning theory are all invalid. An input can be given various labels with different recognition methods. Every input-label pair constitutes a true data sample.  As shown in Figure \ref{pic 1 p1}, the same input contains the labels of ``Apple", ``Red", ``Sweet" and even more. It is a common and typical case. Notwithstanding, it shows the mentioned assumptions are not applicable to general learning problems.
These data samples do not come from an identical distribution since multiple labels violate the normalization of probability, and we even cannot describe the relationship between inputs and labels by a single-value mapping function because an input corresponds to multiple labels. Directly using traditional machine learning to general learning case results in a label confusion, as shown in Figure \ref{pic 1 p3}. In summary, these challenges of real data can be attributed to two essential features of artificial general intelligence (AGI): (1) \emph{Data Complexity.} AGI deals with inconsistent data from uncontrolled disparate tasks and various dynamic environments. (2) \emph{Judgment Complexity.} AGI involves global judgments over a variety of tasks and problems with different regularities\cite{AdamsMapping,Goe2014Artificial,Laird2010Cognitive}. To achieve AGI, we must first thoroughly break the traditional data assumptions and then construct a new representation framework for the real data. Therefore, we propose the subjectivity learning theory.

To construct the representation of real complex data, we introduce a new learning concept --\emph{subjectivity}.
We notice that human actively classifies related judgments of complex data into a specific category, where the data can be  locally represented as a function or distribution. 
From a philosophical perspective, some ideas, conclusions or judgments considered true only from the perspective of a subject \cite{Allen2002Power}. We assign the mathematical meaning to the concept of subjectivity, which is an active division and induction of complex data such that inputs hold a consistent judgment under a certain subject. The machine learns to divide the data into multiple subjects and build judgments for every subject. We refer this novel machine learning method to as subjectivity learning, which can model the complex data cases in general intelligence.
With the introduction of subjects, the machine of subjectivity learning needs to learn two representations: (1) Which subject each sample should be classified into. (2) How to express the data mapping under a certain subject. The main question of subjectivity learning is how to obtain these two descriptions.

To achieve the capabilities of general intelligence, we build a new learning goal -- \emph{the global risk}. We find that human's perception of the world is to avoid fatal errors in any situation, rather than just being accurate in a specific task. It means that the general intelligence adopts a risk metric covering all possible scenarios, which is referred to as \emph{global risk}. In subjectivity learning, this goal is to assess the sum of risks covering all subjects. In this paper, we prove that, for the complex data and the global risk metric, the description in traditional learning theories produces an inevitable error, while the subjectivity learning could mitigate or even eliminate it. 
Therefore, subjectivity learning is more appropriate for general intelligence, and the global risk can drive the representation of subjectivity learning.

In this paper, we propose \emph{subjectivity learning theory} towards general intelligence.
We first describe the framework of subjectivity learning and compare it to traditional learning theory. The principle of empirical global risk minimization is introduced to obtain the practical solution. Then, We further analyze the consistency and the error bound of the learning principle. Our contributions include: 

(1)In general intelligence understanding, we point out the crucial reasons of ``Narrow AI" and violations of data assumptions in traditional learning theory.

(2) In theory, we develop \emph{subjectivity learning} for solving these challenges by introducing the concept of subjectivity. We prove that subjectivity learning can drive a solution with lower risk than traditional machine learning. 

(3) In mathematical method, we extend the Law of Large Number to the case of two coupled variables and prove the consistency of empirical global risk minimization with the increase of samples when certain conditions are satisfied. We further analyzed the error bound and then present triple variables for controlling the error bound.

(4) In philosophy, we attempt to reveal the computational meaning of subjectivity in human intelligence, which explains why the subjectivity is necessary for AI to achieve general intelligence.

\begin{figure}[h]
	\centering
	\begin{minipage}{1\columnwidth}
	\vspace{-0.2cm}
		\subfigure[The Raw Data Samples in AGI Problem.]{ 
			\includegraphics[width = 8.5cm]{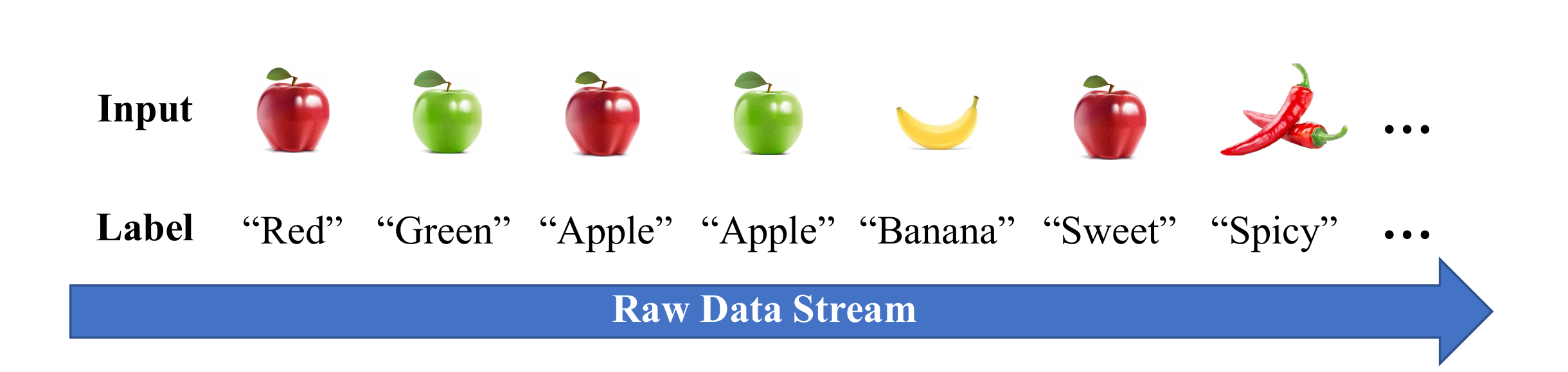}\label{pic 1 p1} 
		}
	\end{minipage} \\
	\vspace{-0.5cm}
	\hspace{-0.5cm}				
	\begin{minipage}{1\columnwidth}
		\vspace{0.4cm}
		\subfigure[Traditional Machine Learning.]{ 
			\includegraphics[width = 8.5cm]{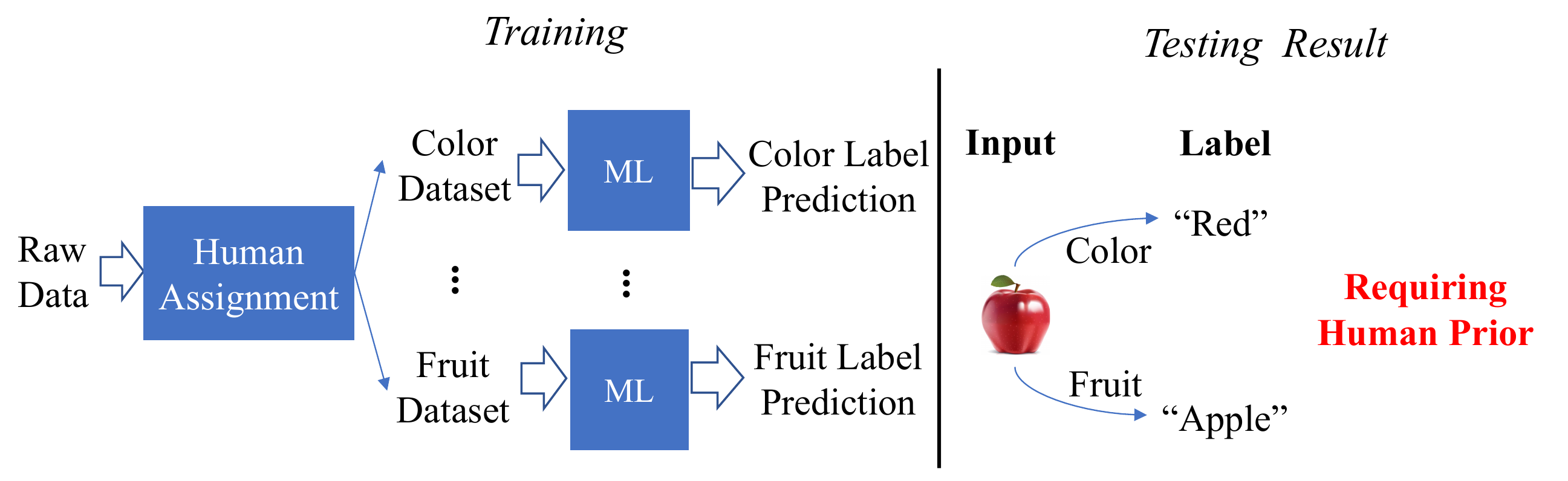}\vspace{2cm}\label{pic 1 p2} 
		}
	\end{minipage} \\
	\hspace{-0.5cm}
	\begin{minipage}{1\columnwidth}
	\vspace{-0.2cm}
		\subfigure[ML without Human-assignment for AGI Problem.]{ 
			\includegraphics[width = 8.5cm]{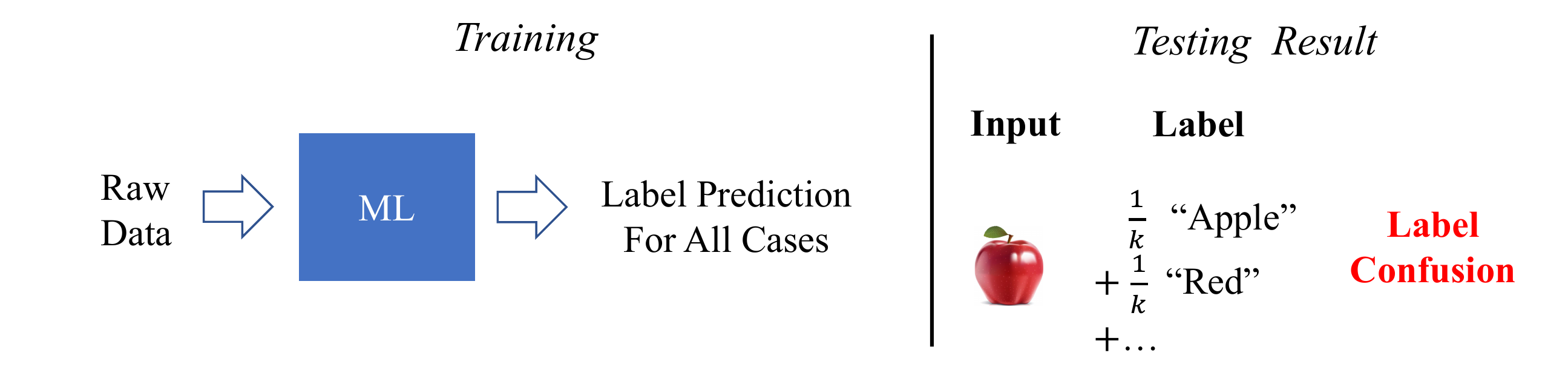}\label{pic 1 p3} 
		}
	\end{minipage} \\
	\hspace{-0.6cm}
	\begin{minipage}{1\columnwidth}
	\vspace{-0.2cm}
		\subfigure[Subjectivity Learning for Solving AGI Problem]{ 
			\includegraphics[width = 8.5cm]{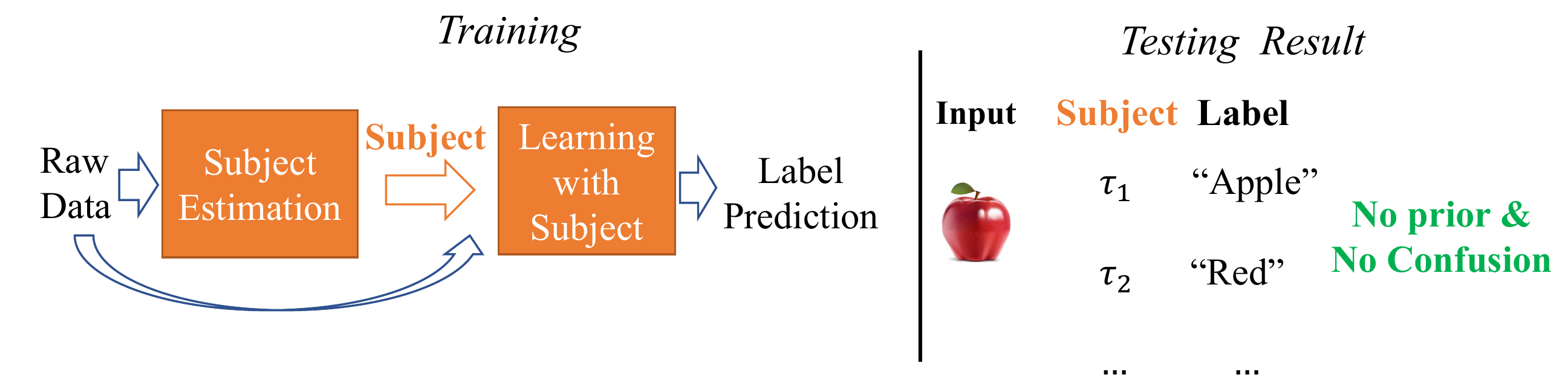}\label{pic 1 p4} 
		}
	\vspace{-0.1cm}
	\end{minipage}
	\vspace{-0.2cm}
	\caption{The data in AGI problem is complex, while every sample only presents partial information. The traditional machine learning focuses on the datasets of a specific problem which is assigned artificially. When they face the AGI problem, serious errors occur. The subjectivity learning actively divides the complex data into multiple subjects and learns the complete representation for AGI.}
	\vspace{-0.4cm}
\end{figure}

\section{Related Work}
Recently, some studies try to break down the limits of data assumptions in traditional learning theory to develop AGI. They can be divided into two categories. The first type of works make efforts to solve the non-i.i.d data challenges in one specific problem. \cite{steinwart2009fast,Yu1994Rates} use various stochastic processes to model the complex dataset. \cite{Balcan2014Efficient} studies classification tasks with unstable distributed data samples in the process of lifelong learning. \cite{pentina2015lifelong} proved that learning tasks with non-i.i.d samples are also beneficial for new tasks.
In these works, although the data samples are not i.i.d, the overall data is still assumed to be a certain distribution, and the input-output relation can be represented by a single-valued mapping function.

The other category focus on the problem of data with multiple distributions. The representative approach is to build a hierarchical architecture, which consists of a task encoder module and a task decoder module \cite{garnelo2018conditional,sung2017learning,schwarz2018progress}. The task encoder module explicitly projects the task-specific dataset to a task vector, and the tasks decoder module predicts targets based on both data inputs and task vectors. \cite{he2019task} pays attention to the task-agnostic continual learning problem, while they assume the data is piece-wise stationary and recent samples are i.i.d. \cite{garnelo2018conditional} addressed learning meta-networks from multiple tasks to perform few-shot learning in supervised learning such as regression, classification and image completion. \cite{sung2017learning} used the hierarchical framework to learn a meta-critic network to perform the few-shot transfer in the domain of reinforcement learning. These efforts attempt to deal with multi-distributed data, while they introduce other artificial assumptions for the raw data. Globally, the single-valued function mapping from the input (data \& task) to the output (label) is still preserved.

In summary, two basic data assumptions in traditional learning theory are not completely broken in all the existing works. We propose the subjectivity learning theory aiming to deal with the general data without these assumptions.

\section{The Framework of Subjectivity Learning}
In this section, we first explain the problem of general intelligence. Then, we define the framework of subjectivity learning clearly and construct the mathematical form of global risk. We further compare the subjectivity learning to traditional statistical learning, and prove that the global risk minimization in subjectivity learning results in a description with a lower risk.

\subsection{Problem Statement}
We consider the common learning scenario of general intelligence. The real data contains various complex cases, while every data sample comes from a specific case or specific evaluation criteria. The label in a sample can only reflect a part of the information in a specific task. Unlike the datasets, the sources and tasks of all data are unknown in AGI problem. The system requires learning from various samples and giving a complete and reasonable representation.

In the traditional learning theory, the data for a specific problem is collected as a dataset. All samples (input-label pair $(x,y)$) were assumed to be independent and identical distributed. The machine is looking for a function $y=f(x)$ (or $F(y|x)$) to express the relation of input $x$ to label $y$ by minimizing the risk functional. Remark the data samples as $z=(x,y)\in Z$ and the mapping function between $x$ and $y$ as $g((x,y)) = g(z)\in G$. When a probability distribution function $F(z)$ is defined on $Z$, the problem of the traditional risk minimization can be expressed as
\begin{align}
	\inf_{g} R_t(g(z))=\int L(z,g(z)) \dif F(z)
	\label{TRM}
\end{align} 
where $L(z,g(z))$ is the loss function of sample $z$. The statistical learning theory ensures that the empirical risk converges to expection with the increase of samples.

However, the samples $(x,y)$ in AGI problem are not i.i.d, also the value of mapping $y=f(x)$ (or $F(y|x)$) changes with various tasks and environments. Human's general intelligence involves how to adjust the judgment according to different environments. Therefore, the current learning theory is not applicable to general intelligence.

\subsection{Subjectivity Learning}

We notice that human's intelligence is based on subjectivity in making specific decisions, and one thing may correspond to different judgments under different subjects. Therefore, we draw on the concept of subjectivity to deal with complex data. In philosophical, that is the collection of the perceptions, experiences, expectations, and beliefs specific to a person. We define the mathematical meaning of subjectivity that
\begin{definition}
	The subjectivity is defined as the subjective collection for data samples with unified mapping, distribution, and loss metrics. 
\end{definition}

The core idea of subjectivity learning is learning to subjectively divide complex data samples into various subjects and to represent their various mappings. Although this method can deal with AGI data, it brings a new variable that is the subject attribution of the sample. 

Specifically, we remark the subject as $\tau$. The data description in subjectivity learning includes two parts:\\
 (1) \emph{What is the input-label mapping $y=f(x,\tau)$ (or $F(y|x,\tau)$) under a specific subject $\tau$}? It is similar to the function $y=f(x)$ in the traditional machine learning, but this relation can only be expressed as a function under a single subject.\\
 (2) \emph{Which subject $\tau$ should the samples $(x,y)$ belong to}? It's a new concept in the subjectivity learning. Mathematically, when we give the subjects attribution for data samples, the data and subjects form a joint distribution $F((x,y),\tau)$. Different from the traditional learning problem, this joint distribution is changing in the learning process. Thus, the sample attribution corresponds to the posterior probability for subjects, that is $p(\tau|(x,y))$, which is a function variable to be learned. \\
The current question is what is the goal driving the representation of subjectivity learning.
 

\subsection{Global Risk Functional}
The goal of human's intelligence is to avoid the fatal error in almost any cases, instead of only focusing on the risk of specific tasks. We also adopt this goal as the evaluation and construct the second core concept, \emph{the global risk functional}, for subjectivity learning. Since the data samples are divided by different subjects in subjectivity learning, the goal should consider the loss over all subjects.

Mathematically, when given the sample $z=(x,y)$, the subject of $z$ is remarked as $h(z,\tau)=p(\tau|z)/p(\tau)$. The input-label relation under a specific subject $\tau$ is defined as $g(z,\tau)=F(y|x,\tau)$. We define that
\begin{definition}
	The \textbf{global risk functional} in the subjectivity learning is defined as
	\begin{align}
	\inf_{g,h} R_s(g,h)=\int L_0(z,\tau,g(z,\tau))h(z,\tau) \dif F(z) \dif F(\tau),
	\label{GRM}
	\end{align}
	where $L_0$ is the loss function of sample $z$ under subject $\tau$, $F(z)$ and $F(\tau)$ are the distributions of sample $z$ and $\tau$.
\end{definition}

This global risk is related to the representation of subjectivity learning. For $g(z,\tau)$, it is obvious that the risk increases if the output does not match the real label. For the $h(z, \tau)=p(\tau|z)/p(\tau)$, it also causes a serious error if conflict samples in different nature are classified as the same subject. The following sample can clearly illustrate this situation.

Let the data (input-label pairs) be $(x,y)$. For traditional statistical learning problem, let the output of the learning machine from input $x$ as $f(x)$. Then the loss function is formed as $L(f(x),y)$, and the risk function should be expressed as
\begin{align*}
R_t(f) = \int L(f(x),y)dF(x,y)
\end{align*}
Note that an sample $x$ may correspond to multiple $y_j$ in general case. If we optimize the risk functional 
\begin{align*}
R_t(f) = \sum_{i=1}^{l} L(f(x_i),y_i)
\end{align*}
over all data, the optimal solution is that
\begin{align*}
f(x_i)= \bar{y}_{i} = \frac{1}{m} \sum_{j=1}^{m} y_{ij}.
\end{align*}
And there exist an absolute confusion error in the optimal loss, that is
\begin{align*}
\Delta R = \sum_{i,j} L( \bar{y}_{i}, y_{i,j}).
\end{align*}
It is the label confusion shown in Figure1(c) of the main context.
In the subjectivity learning, the samples are divided into different subjects $\tau$ and measured by global risk $R_s(g,h)$. If the subjects' division $h(z,\tau)$ is unreasonable, such as all samples are still classified into one subject, the above absolute error is reflected in the global risk. 
On the other hand, when samples are reasonably divided into different subjects, such as no conflict of sample mapping in any subject, the global risk
\begin{align*}
R_s=\sum_{j} \sum_{i} L(f(x_i,\tau_j,\alpha),x_i)I\left [ (x_i,y_i)\in \tau_j \right ]
\end{align*}
likely converges to zero. It also shows that subjectivity is pretty valuable to general intelligence instead of only human's prejudice against things.

Subjectivity representations and the global risk functional combined role makes the general intelligence a learning problem. Next, we compare the subjectivity learning with traditional machine learning, and prove that subjectivity learning results in a lower global risk.


\subsection{Risk Gap}
Here we compare the global risk minimization of subjectivity learning to traditional learning problem. The following theorem demonstrates that the subjectivity learning results in a lower optimal risk under the equivalent loss measure.
\begin{theorem}
	Let us consider the problem of machine learning (\ref{TRM}) and the problem of subjectivity learning (\ref{GRM}). Under the equivalent loss measure $L(z,g(z))|\tau = L_0(z,\tau,g(z))$, the inequality
	\begin{align}
	     \inf_{g_1} R_{t}(g_1(z)) \geqslant \inf_{g_2,h} R_{s} (g_2,h)
	     \label{th1 p1}
	\end{align}
	take place. 
	
	For the optimal solution $g^*(z,\tau)$ and $h^*(z,\tau)$, if there exists samples measured with $\dif F(z,\tau_1),\dif F(z,\tau_2)>0$ such that $g^*(z,\tau_1) \neq g^*(z,\tau_2)$, there exist an absolute risk gap that
	\begin{align}
	\Delta R^* = \inf_{g} R_t(g(z)) - \inf_{g,h} R_s(g,h) >0
	\label{th1 p2}
	\end{align}
\end{theorem}
\begin{proof}
	See the Supplementary Material.
\end{proof}

The above theorem contains two parts. In the first part  (\ref{th1 p1}), we qualitatively state that the global risk functional of subjectivity learning drives a more accurate description. In AGI problem, a sample $z$ almost holds multiple different judgment $g(z,\tau)$ with various $\tau$. In the second part (\ref{th1 p2}), we prove that there must be a risk gap greater than 0 between the traditional machine learning and the subjectivity learning. Statistical learning theory ensures that minimum empirical risk approaches the lower bound of the expected risk with the increase of data samples. However, the above risk gap demonstrates that the lower bound of traditional risk functional can never approach the optimal description of AGI. On the other hand, the subjectivity learning with the global risk is able to drive a better description. From a philosophical perspective, this theorem also explains why the human constructs various subjectivity for the real complex data, which is a method of interpreting the world with less risk.

 After presenting the framework of subjectivity learning, we then illustrate the learning process and its convergence in next section.


\section{The Theory of Subjectivity Learning}
After giving the representation of subjectivity learning and the form of global risk functional, the current main question is how to find the minimization of the global risk functional. In this section, we introduce the principle of \emph{empirical global risk minimization (EGRM)}, and explain the process of subjectivity learning under this principle. Different from statistical learning theory, the subjectivity learning process involves two types of samples, which are data samples $z$ and subjects samples $\tau$, and the numbers of them are related. The nature of convergence changes with this relationship. Therefore, we generalize the Law of Large Number to the case of coupled variables. Then, we establish the consistency conditions and give the convergence probability of the learning process. We prove that the empirical global risk minimization can tend to the expectation with the increase of data samples and subjects samples.

\subsection{Principle of EGRM}

We cannot directly minimize the functional (\ref{GRM}) since the probability distributions of $F(z)$ in the definition is unknown. Also, we need to consider the newly introduced variable, subject $\tau$. In practical, we first get the data samples $z_1,...,z_l$ in various cases.  The form of subjects $\tau_1,...,\tau_m$ is a set of samples from a prior distribution $F(\tau)$, which can be a language, a symbol or something else. Different with data samples, subject samples is artificially introduced and controllable. Before constructing the joint distribution with data samples, the subject sample do not have a specific physical meaning, so it can be sampled from any certain prior distribution. However, the number of subject samples is related to the giving data samples, and it should change as new samples continually arrive. When the data samples and subject samples are determined, we need to select the optimal functions $g$ and $h$ from the sets $G$ and $H$ to minimize the global risk. Therefore, the learning process of subjectivity learning can be summarized as first sampling a set of subjects $\tau_1,...,\tau_m$ on the basis of data samples $z_1,...,z_l$, and then selecting the optimal description functions $g$ and $h$ based on these samples to make the global risk functional minimum. This novel principle is called \emph{empirical global risk minimization (EGRM)}. Now we show the specific form in mathematical.

For simplicity, we first rewrite the global risk functional. When formulating the minimization of the functional (\ref{GRM}), the set of functions $g(z,\tau)$ and $h(z,\tau)$ will be given in a parametric form that $\{ g(z,\tau;\alpha_g) \alpha_g \in \Lambda_g \}$ and $\{ h(z,\tau;\alpha_h) \alpha_h \in \Lambda_h \}$. Here $\alpha_g$ and $\alpha_h$ are parameters from the set $\Lambda_g$ and $\Lambda_h$ such that the value $\alpha_g=\alpha_g^*$ defines the specific function $g(z,\tau;\alpha_g^*)$ in the set $g(z,\tau)$ and similar as $\alpha_h$. We further merge these two sets of parameters such that $\alpha = (\alpha_g,\alpha_h)$ and $\alpha \in \Lambda = \Lambda_g \times \Lambda_h$. In this notation, the functional (\ref{GRM}) can be written as
\begin{align}
	\inf_{\alpha \in \Lambda} R(\alpha)=\iint Q(z,\tau,\alpha) \dif F(z) \dif F(\tau)
	\label{GRM2}
\end{align}
where
\begin{align}
	Q(z,\tau,\alpha) = L_0 (z,\tau, g(z,\tau;\alpha_g))h(z, \tau;\alpha_h).
\end{align}
The function $Q(z,\tau,\alpha)$, which depends on variables $z$, $\tau$ and $\alpha$, is called basic loss function. Each function $Q(z,\tau,\alpha^*),\alpha^*\in \Lambda$ determines the value of the loss resulting from the data vector $z$ and subject vector $\tau$.

Then we introduce the principle of the empirical global risk minimization that
\begin{definition}
        \textbf{(Principle of Empirical Global Risk Minimization, EGRM)} On the basis of data samples $z_1,...,z_l$, we select a suitable number of subject samples $\tau_1,...,\tau_m$ and minimize the functional
	\begin{align}
		R_{emp}(\alpha, m, l)=\frac{1}{m} \sum_{j=1}^{m} \frac{1}{l} \sum_{i=1}^{l} Q(z_i,\tau_j,\alpha),\ \alpha \in \Lambda,
		\label{EGRMeq}
	\end{align}
	which is called empirical global risk functional.
\end{definition}

This functional is defined in an explicit form, and can be minimized. Let the minimum of the global risk functional be attained at $Q(z,\tau,\alpha_0)$ and let the minimum of the empirical global risk be attained at $Q(z,\tau,\alpha_{l,m})$. We take the principle of EGRM that is using the function $Q(z,\tau,\alpha_{l,m})$ as an approximation of the function $Q(z,\tau,\alpha_0)$.  The next problem is to establish the conditions under which the function $Q(z,\tau,\alpha_{l,m})$ is close to the function $Q(z,\tau,\alpha_0)$. Also, we want to know (1) how to control the number of subject number with data samples; (2) how the speed of $Q(z,\tau,\alpha_{l,m})$ is close to $Q(z,\tau,\alpha_0)$ as the data and subjects samples increases. These questions are discussed in the following context.

\subsection{Convergence with Two Coupled Variables}

We have proposed the empirical global risk to approximate the expectation (\ref{GRM2}). We need to determine under which conditions such an approximation is valid. In this subsection, we first give the definition of consistency. The consistency in statistical learning theory is based on the Law of Large Number that is the experience converges to expectation as the number of samples increasing. While in the subjectivity learning, besides the data sample, subjects sample are newly introduced. The convergence should consider the increase of these two type of variables and their relationship. Therefore, we generalize the Law of Large Number to the case of two coupled variable and define the consistency for the principle of EGRM.

Let us consider a related empirical process. Let the probability distribution function $F(\tau)$ and $F(z)$  be defined on the space $\tau \in \mathbb{R}^{n_\tau}$ and $z \in \mathbb{R}^{n_z}$, and let $Q(z,\tau,\alpha), \alpha \in \Lambda$ be a set of measurable loss functions. Let $\tau_1,...,\tau_m,...$ and $z_1,...,z_{l},...$ be sequences of independent identically distributed vectors of subjects and data. Consider the one-sided empirical process given by the sequence of random values 

\begin{align}
\begin{split}
\xi ^{\left < m,l \right >} = \sup _{\alpha \in \Lambda}  \left(  R(\alpha)-R_{emp} (\alpha, m, l) \right)  \\
\left < m,l \right >=&1,2,...
\end{split} \
\label{process1}
\end{align}
which $R(\alpha)$ is the form (\ref{GRM2}) and $R_{emp}$ is the form (\ref{EGRMeq}).

In this process, we need to consider the increase of samples number $m$ and $l$ simultaneously, since different number order of two samples changes the characteristics of the process and its convergence. We note it as $\left <m,l \right >$, which means the growth of these two variables is based on a certain rule, and $\left < m,l \right > \rightarrow \infty $ notes the variables $m,l$ both tend to infinity under this rule. In the subjectivity learning, this rule can  be adjusted with controllable subjects number. 

The Law of Large Numbers demonstrates that the sequence of means converges to expectation of a random variable (if it exists) as the number of samples increases. As the starting point for consistency theory, we first generalize the Law of Large Number to the case of two couple variables.  Now we introduce the theorem: 
\begin{theorem}
	 (\textbf{Convergence Theorem with two coupled variables.}) When the $\alpha^*$ is determined, for the function $Q(z,\tau,\alpha^*)$ and any $\varepsilon > 0$, the following convergence
	\begin{align}
	P\Big \{  R(\alpha^*)-R_{emp} (\alpha^*, m, l) >\varepsilon  \Big \} \xrightarrow[\left < m,l \right >\rightarrow \infty]{} 0
	\label{eq lemma 1}
	\end{align}
	take place, where the sample numbers $\left <m,l \right >$ satisfied the rule:
	\begin{align}
		l > \frac{2(B_z - A_z)^2}{\varepsilon ^2} \ln m + \frac{(B_z - A_z)^2}{(B_\tau - A_\tau)^2} m.
		\label{sample relation}
	\end{align}
	The $B_z,A_z$ and $B_\tau, A_\tau$ are respectively the bound of functions that $A_z \leq Q(z,\tau) \leq B_z$ and $A_\tau \leq R^{lo}(\alpha, \tau) \leq B_\tau$.
	\label{lemma 1}
\end{theorem}
\begin{proof}
	See the Supplementary Material
\end{proof}
In the condition of Theorem \ref{lemma 1}, $R^{lo}(\alpha, \tau)$ is the local risk that
\begin{align}
R^{lo}(\alpha, \tau) = \int Q(z,\tau,\alpha)dF(z),	
\end{align}
which represents the risk integral of all samples $z$ under the specific subject $\tau$. It is only related to the subject variable $\tau$ and parameter $\alpha$. The Theorem \ref{lemma 1} shows that the sequence of $\xi ^{m,l}$ always converges in probability to zero, if set of functions $Q(z,\tau,\alpha),\alpha \in \Lambda$ contains only one element, that is the function $Q(z,\tau,\alpha)$ is determined. The consistency of subjectivity learning should consider the set of functions contains multiple and even infinite elements. Now we give the definition of consistency.

\begin{definition}
	We say that the method of global empirical risk minimization is strictly (non-trivially) consistent the set of function $Q(z,\tau,\alpha),\alpha \in \Lambda$ if for any nonempty subset $\Lambda(c), c\in (-\infty, \infty)$ of this set of functions such that
	\begin{align}
	\Lambda(c) = \{\alpha : \iint Q(z,\tau,\alpha) \dif F(z)\dif F(\tau) \geq c \},
	\end{align}
	the next convergence is valid:
	\begin{align}
	\inf_{\alpha \in \Lambda(c)} R_{emp}(\alpha,m,l)  \xrightarrow[\left < m,l \right >\rightarrow \infty]{P}   \inf_{\alpha \in \Lambda(c)} R(\alpha)
	\label{eq consistency}
	\end{align}
\end{definition} 

Our goal is to find the conditions to make consistency (\ref{eq consistency}) exist. In the derivation, we use the convergence conditions of the process (\ref{process1}) to construct the conditions of consistency, that is to describe conditions such that for any $\varepsilon > 0$, the following relation
\begin{align}
P\left\{  \sup _{\alpha \in \Lambda}  \left(  R(\alpha)-R_{emp} (\alpha, m, l) \right) > \varepsilon \right\} \xrightarrow[\left < m,l \right >\rightarrow \infty]{} 0
\label{one-sided convergence}
\end{align}
takes place. This formula is referred to one-sided uniform convergence.





\subsection{Conditions of Consistency}

We first show that one-sided uniform convergence (\ref{one-sided convergence}) forms not only the sufficient conditions for the consistency of the EGRM, but the necessary conditions as well. We further generalize Theorem \ref{lemma 1} to the case of function set $Q(z,\tau,\alpha),\alpha \in \Lambda$ and construct the conditions. 

We formulate the following key theorem of subjectivity learning theory to describe the above conclusions, similar to the equivalent theorem of statistical learning theory.

\begin{theorem}
	\textbf{(the Equivalent Theorem)} Let there exist the constants $a$ and $A$ such that for all functions in the set $Q(z,\tau,\alpha),\alpha \in \Lambda$ and for distribution functions $F(t)$ and $F(z)$, the inequalities
	\begin{align}
	a \leq \iint Q(z,\tau,\alpha) \dif F(z)\dif F(\tau) \leq A
	\end{align}
	hold true. Then the following two statements are equivalent:
	
	1. The empirical global risk minimization method is strictly consistent (\ref{eq consistency}) on the set of functions $Q(z,\tau,\alpha)$.
	
	2. The uniform one-sided convergence of the means to their mathematical expectation (\ref{one-sided convergence}) takes place over the set of functions $Q(z,\tau,\alpha)$.
\end{theorem}

\begin{proof}
	See the Supplementary Material
\end{proof}

This theorem transforms the problem of consistency to the problem of one-side uniform convergence. Now, we describe the conditions for one-side uniform convergence (\ref{one-sided convergence}). With the local risk, the following inequality is valid:
\begin{theorem}
	For any $\varepsilon>0$, the following inequality holds:
	\begin{align}
	\begin{split}
	&P\Big \{  \sup _{\alpha \in \Lambda}  \big (  R(\alpha)-R_{emp} (\alpha, m, l) \big ) > \varepsilon \Big \} \\
	\leq & P\Big\{  \sup _{\alpha \in \Lambda}  \big (  \int R^{lo}(\alpha,\tau)\dif F(\tau) -\frac{1}{m}\sum_{j=1}^{m} R^{lo}(t_j,\alpha) \big ) > \varepsilon \Big \} \\
	& + \sum_{j=1}^{m} P\Big\{  \sup _{\alpha \in \Lambda}  \big (  \int Q(z,\tau_j,\alpha)\dif F(\tau) \\
	&\qquad \qquad\qquad- \frac{1}{l}\sum_{i=1}^{l} Q(z_i,\tau_j,\alpha) \big ) > \varepsilon \Big \}
	\end{split}\label{pro eq1}
	\end{align}
\end{theorem}

\begin{proof}
	See the Supplementary Material.
\end{proof}

From the above theorem, the convergence probability consists of two terms, where the first term is the convergence probability of the subjectivity risk and the second term is the sum of probability of data risk under all subjects. We further use the concept of capacity like statistical learning theory to discuss the conditions of uniform convergence.

For the first term, let $R^{lo}(\tau,\alpha),\tau \in T, \alpha \in \Lambda$ be a set of real-valued functions. Let $N_\tau^{\Lambda,\beta_\tau}(\tau_1,...,\tau_m)$ be the number of different separations of $m$ vectors $\tau_1,...,\tau_m$ by a complete set of indicators (detailed in supplementary material). Then we define the annealed entropy of subjectivity risk that

\begin{definition}
	(\textbf{Annealed Entropy of Subjectivity Risk}) The quantity
	\begin{align}
		\hat{H}_{\tau}^{\Lambda,\beta_\tau}(m) = \ln E N_\tau^{\Lambda,\beta_\tau}(\tau_1,...,\tau_m)
	\end{align}
	is defined as the annealed entropy of the set indicators of real-valued functions $R^{lo}(\tau,\alpha)$.
\end{definition}
Using the error inequality in statistical learning theory\cite{Vapnik2003Statistical}, for the bounded real-valued functions $A_\tau \leq R^{lo}(\tau,\alpha)\leq B_\tau,\alpha \in \Lambda$, the following inequality is valid:
\begin{align}
\begin{split}
	P\Big\{  \sup _{\alpha \in \Lambda}  \big (  \int & R^{lo}(\alpha,\tau) \dif F(\tau) -\frac{1}{m}\sum_{j=1}^{m} R^{lo}(\tau_j,\alpha) \big ) > \varepsilon \Big \} \\
& \leq 4exp\Big \{ \Big ( \frac{\hat{H}_{\tau}^{\Lambda,\beta_\tau}(2m)}{m} - \frac{(\varepsilon-\frac{1}{m})^2}{(B_\tau-A_\tau)^2} \Big )m \Big \}.
\label{bound eq1}
\end{split}
\end{align}


Also, we define the annealed entropy of data risk for   $Q(z,\tau,\alpha),z\in Z,\alpha \in \Lambda$ and consider the second term of (\ref{pro eq1}). Let $N_z^{\Lambda,\beta_z}(z_1,...,z_l)$ be the number of different separations of $l$ vectors $z_1,...,z_l$ by a complete set of indicators (detailed in supplementary material). We define that

\begin{definition}
	(\textbf{Annealed Entropy of Data Risk})  The quantity
	\begin{align}
	\hat{H}_{z}^{\Lambda,\beta_z}(l) = \ln E N_z^{\Lambda,\beta_z}(z_1,...,z_l)
	\end{align}
	is defined as the annealed entropy of the set indicators of real-valued functions $Q(z,\tau,\alpha)$ under a specific $\tau$.
\end{definition}
For the bounded real-valued functions $A_z \leq Q(z,\tau,\alpha) \leq B_z, \alpha \in \Lambda$, the following inequation is valid:
	\begin{align}
	\begin{split}
	&\sum_{j=1}^{m} P\Big\{  \sup _{\alpha \in \Lambda}  \big (  \int Q(z,\tau_j,\alpha)\dif F(\tau) \\
	&\qquad \qquad\qquad\qquad- \frac{1}{l}\sum_{i=1}^{l} Q(z_i,\tau_j,\alpha) \big ) > \varepsilon \Big \}
	\end{split}\\
	\begin{split}
	\leq & 4exp\Big \{ \Big ( \frac{\ln m}{l} + \frac{\hat{H}_{z}^{\Lambda,\beta_z}(2l)}{l} - \frac{(\varepsilon-\frac{1}{l})^2}{(B_z-A_z)^2} \Big )l \Big \}
	\end{split}
	\label{bound eq2}
	\end{align}




Let substitute the inequation (\ref{bound eq1}) and (\ref{bound eq2}) into (\ref{pro eq1}), we get:
\begin{theorem}
	Let $A_\tau\leq R^{lo}(\tau,\alpha)\leq B_\tau, \alpha \in \Lambda$ and $A_z \leq Q(z,\tau,\alpha) \leq B_z,\alpha \in \Lambda$ be measurable set of bounded real-valued functions. Let $\hat{H}_{\tau}^{\Lambda,\beta_t}(m)$ and $\hat{H}_{z}^{\Lambda,\beta_z}(l)$ be the annealed entropies of the sets of indicators for them. Then the following inequality is valid:
	\begin{flalign}
		\begin{split}
		&P\Big \{  \sup _{\alpha \in \Lambda}  \big (  R(\alpha)-R_{emp} (\alpha, m, l) \big ) > \varepsilon \Big \} \\
		 \leq& 4exp\Big \{ \Big ( \frac{\hat{H}_{\tau}^{\Lambda,\beta_\tau}(2m)}{m} - \frac{(\varepsilon-\frac{1}{m})^2}{(B_\tau-A_\tau)^2} \Big )m \Big \} \\
		& 4exp\Big \{ \Big ( \frac{\ln m}{l} + \frac{\hat{H}_{z}^{\Lambda,\beta_z}(2l)}{l} - \frac{(\varepsilon-\frac{1}{l})^2}{(B_z-A_z)^2} \Big )l \Big \}.
		\end{split}\label{th of cons eq}
	\end{flalign}
	\label{th of cons}
	\vspace{-0.3cm}
\end{theorem}

Note that the samples number satisfied the inequality (\ref{sample relation}) makes $\lim_{l,m \rightarrow \infty} {\ln m}/{l} = 0 $ must be true. Therefore, from the above theorem, we can establish a set of sufficient conditions for the uniform convergence.
\begin{corollary}
	\textbf{(Sufficient Conditions of Consistency)} For the existence of non-trival exponential bounds on uniform convergence, the sufficient conditions is to satisfy all the following three formulas: 
	\begin{flalign}
		 & \lim_{l \rightarrow \infty} \frac{\hat{H}_{z}^{\Lambda,\beta_z}(l)}{l} = 0 \label{co eq 1}\\
		 & \lim_{m \rightarrow \infty} \frac{\hat{H}_{\tau}^{\Lambda,\beta_t}(m)}{m} = 0 \label{co eq 2}\\
		 & l > \frac{2(B_z - A_z)^2}{\varepsilon ^2} \ln m + \frac{(B_z - A_z)^2}{(B_\tau - A_\tau)^2} m \label{co eq 3}
	\end{flalign}
	\vspace{-0.3cm}
\end{corollary}

It is the sufficient condition for one-side uniform convergence (\ref{one-sided convergence}), and is also sufficient conditions for the consistency of EGRM. The condition consists of three parts.  The equation (\ref{co eq 1}) means that under the specific subject $\tau$, the number of distinguishable events $N_z^{\Lambda, \beta_z}$ should increase slowly as the data sample size increases (slower than any exponential function). The equation (\ref{co eq 2}) considers the local risk $R^{lo}(\tau,\alpha)$ of different subjects. It requires that the number of distinguishable events $N_\tau^{\Lambda, \beta_\tau}$ for local risk increases slowly as the subject sample size increases (slower than any exponential function).  Besides, the equation (\ref{co eq 3}) constraints the number relation between the subjects and data samples. So far, we have established a sufficient condition for consistency. Next, we analyze the error bound of global risk and discuss how to control the global risk in the case of determined number of data samples.


\subsection{Triple Variables for Global Risk Controlling}
The Theorem \ref{th of cons} shows the probability of uniform convergence , which is also the probabilistic form on generalization ability. In this subsection, we further analyze the constructive distribution-free bounds on generalization ability, and propose triple variables for controlling the global risk.

For analyze the inequality (\ref{th of cons eq}), we introduce the concept of the \emph{data dimension} $h_z$ and the \emph{subject dimension} $h_\tau$ for subjectivity learning, which are similar to the VC dimension for statistical learning theory. The data dimension (subject dimension) of a set of indicator functions $Q(z,\tau,\alpha),\alpha \in \Lambda$( or $R^{lo}(\tau,\alpha),\alpha\in\Lambda$) is equal to the largest number $h_z$ (or $h_\tau$) of vectors $z_1,...,z_l$ (or $\tau_1,...,\tau_m$) that can be shattered by the complete set of indicators. 
These two dimensions satify that:
\begin{align}
	\hat{H}_{z}^{\Lambda,\beta_z}(l) \leq h_z\Big ( \ln \frac{l}{h_z} + 1 \Big ) \\
	\hat{H}_{\tau}^{\Lambda,\beta_\tau}(m) \leq  h_\tau\Big ( \ln \frac{m}{h_\tau} + 1 \Big ).
\end{align}
We take them into (\ref{th of cons}) and rewrite it into the form of error bound. We have that
\begin{theorem}
	 With probability $1-\eta$ the risk for the function $Q(z,t,\alpha_{l,m})$ which minimizes the empirical glob risk functional satisfies the inequality
	\begin{align}
	R(\alpha_{l,m}) < R_{emp}(\alpha_{l,m}) + \varepsilon_{l,m},
	\label{error bound}
	\end{align}
	where $\varepsilon_{l,m}$ satisfies 
	\begin{align*}
	\eta &= 4exp\Big \{ \Big ( \frac{h_t}{m}(1+\ln \frac{2m}{h_\tau}) - \frac{(\varepsilon_{l,m}-\frac{1}{m})^2}{(B_\tau-A_\tau)^2} \Big )m \Big \} \\
	&\  + 4exp\Big \{ \Big ( \frac{\ln m}{l} + \frac{h_z}{l}(1+\ln \frac{2l}{h_z}) - \frac{(\varepsilon_{l,m}-\frac{1}{l})^2}{(B_z-A_z)^2} \Big )l \Big \}
	\end{align*}
\end{theorem}

After getting the bound of generalization error, now we consider how to control the error bound (\ref{error bound}) when the size of data samples $l$ is small. In the statistical learning theory, this issue is discussed by structural risk minimization principle and is controlled by VC dimension. In the subjectivity learning, there are three related factors: data dimension $h_z$, subjects number $m$, and subject dimension $h_l$. 
When $l$ is determined, we first adjust the data dimension $h_z$ and subjects number $m$ to balance the empirical global risk and error of generalization. For EGRM, the smaller number of subjects and small local dimensions could reduce the error of generalization, while they result in a higher empirical global risk. After the subjects number $m$ is determined, the error bound is related to subjectivity dimension $h_\tau$. 
Therefore, for controlling the error bound of subjectivity learning, there are two crucial difference: (1) Besides the design of function complexity $h_z$ and $h_\tau$, it is necessary to control the number of subject samples $m$ to balance the empirical global risk and generalization error. (2) The subjectivity dimension $h_\tau$, reflecting the complexity of subjectivity representation, should change with the number of subject samples. It means that the structure related to subjectivity dimension in the learning machine also need to adjusted dynamically. So far, we have given the complete theory of subjectivity learning.






\section{Conclusion}
In this paper, we point out two basic data assumptions in the current AI and machine learning methods, which are not applicable to the complex data in general intelligence. We thoroughly break these assumptions and develop the theory of subjectivity learning. We make a try to introduce the mathematical meaning to subjectivity, which is the concept of traditional philosophy. The introduction of subjectivity makes it possible to describe complex real data for general intelligence. Our theory proves the feasibility of subjectivity learning framework and raises the guiding idea for AGI algorithm in the future. However, there are still many difficulties in implementation, such as how to design the algorithm structure to express the functions in subjectivity learning and how to optimize them. These questions will be answered in the following works. Beyond theory, we also believe that there exists a physiological explanation for subjectivity learning, and subjectivity learning theory can also model related physiological phenomena in human intelligence. Although we have not yet reached a complete interpretation of general intelligence, the idea of subjectivity learning provides a valuable direction to solve intelligence puzzles.

\bibliographystyle{aaai}
\bibliography{egbib}
\clearpage

\begin{center}
	\textbf{\LARGE Supplementary Material}
\end{center}

Due to the limitations of the length of the paper, we put the proof of theorems and some details of discussion in this supplementary Material.

\section{Theorem Proof in \\ The Framework of Subjectivity Learning}

In the Framework of Subjectivity Learning, we consider the general learning scenario and give the form of traditional machine learning and subjectivity learning. The learning system is given a set  of  input-label pairs $(x_i, y_i)$. In the statistical learning theory, it was assumed that all samples are independent and identical distributed. It looks for a function $y=f(x)$ (or $F(y|x)$)by minimizing the risk function. We remark samples as $z=(x,y)\in Z$ and remark $y=f(x)$ (or $F(y|x)$) as $g((x,y)) = g(z)\in G$. The traditional risk minimization can be written as 
\begin{align}
\inf_{g} R_t(g(z))=\int L(z,g(z)) \dif F(z)
\label{TRM}
\end{align} 
where $L(z,g(z))$ is the loss function of sample $z$ and function $g(z) \in G$ is on the function space $Z\rightarrow \mathbb{R}$.

Note that the sample pairs of real data do not meet the independent and identical distributed. They may come from multiple independent distributions (e.g, the mapping $f(x_i)=y_{i,1}$ and $f(x_i)=y_{i,2}$ are both right with probability 1 but $y_{i,1}\neq y_{i,2}$). It is obviously wrong to directly estimate one posterior probability $p(y|x)$ since it does not satisfy the normalization condition that $\sum_{y} p(y|x) \neq 1$. Even the mapping from $x$ to $y$ can not be expressed as a function.

To describe this complex data, we introduce the concept of subjectivity. The data are subjectively divided into multiple subjects $\tau$ and construct a joint distribution $F(z,\tau)$. Under a specific subject $\tau$, the input $x$ contains a unique output $y$ and it could be expressed as a function $y=f(x,\tau)$ or $F(y|x, \tau)$. At  the same time, we add a new variable, which is the subject attribution of the sample $p(t|z)$. This framework is named subjectivity learning. The goal of subjectivity learning is to learning the variables $g(z,\tau)=F(y|x,\tau)$ and $h(z, \tau)= p(\tau|z)/p(\tau)$. As shown in the paper, we construct the \emph{global risk functional} that:
\begin{align}
\inf_{g,h} R_s(g,h)=\int L_0(z,\tau,g(z,\tau))h(z,\tau) \dif F(z) \dif F(\tau)
\label{GRM}
\end{align}
as the learning goal of subjectivity learning.

\subsection{Proof of Theorem 1}

We compare the traditional risk of statistical learning problem to the global risk of subjectivity learning. The Theorem 1 demonstrates that the minimization of global risk results in a lower optimal risk under the equivalent loss measure, also there exists a positive risk gap.
\begin{theorem}
	Let us consider the problem of machine learning (\ref{TRM}) and the problem of subjectivity learning (\ref{GRM}). Under the equivalent loss measure $L(z,g(z))|\tau = L_0(z,\tau,g(z))$, the inequality
	\begin{align}
	\inf_{g_1} R_{t}(g_1(z)) \geqslant \inf_{g_2,h} R_{s} (g_2,h)
	\label{th1 p1}
	\end{align}
	take place. 
	
	For the optimal solution $g^*(z,\tau)$ and $h^*(z,\tau)$, if there exists samples measured with $\dif F(z,\tau_1),\dif F(z,\tau_2)>0$ such that $g^*(z,\tau_1) \neq g^*(z,\tau_2)$, there exist an absolute risk gap that
	\begin{align}
	\Delta R^* = \inf_{g} R_t(g(z)) - \inf_{g,h} R_s(g,h) >0
	\label{th1 p2}
	\end{align}
\end{theorem}

\begin{proof}
	
	The theorem contains two parts. We firstly qualitatively state that the global risk functional drives a lower risk, and then give the proof of the positive risk gap.
	
	We consider a set of samples $z_1,...,z_l,...$ are from the distribution $F(z)$. The traditional risk minimization is defined as 
	\begin{align}
	\inf_{g_1(z)\in G_1} R_t(g_1(z))=\int L(z,g_1(z)) \dif F(z)
	\label{th0 proof eq1}
	\end{align}
	where $G_1 = Z\rightarrow \mathbb{R}$.
	For comparison to the global risk minimization of subjectivity learning, we first consider a certain data-subjects division $\hat{h}(z,\tau)$. We expand the risk function (\ref{th0 proof eq1}) under this joint distribution that:
	\begin{align*}
	& \inf_{g_1(z)\in G_1} R_t(g_1(z)) \\
	=&\inf_{g_1(z)\in G_1} \int [R_t(g_1(z))\mid \tau] \cdot \dif F(\tau) \\
	=&\inf_{g_1(z)\in G_1} \iint [L(z,g_1(z))\mid \tau] \cdot \dif F(z|\tau) \dif F(\tau).
	\end{align*}
	By the condition $L(z,g(z))|\tau = L_0(z,\tau,g(z))$ and $h(z,\tau)=p(\tau|z)/p(\tau)$, the above risk minimization can be expressed as 
	\begin{align}
	\begin{split}
	& \inf_{g_1(z)\in G_1} R_t(g_1(z))  \\
	=&\inf_{g_1(z)\in G_1} \iint L_0(z,\tau,g(z))\hat{h}(z,\tau) \dif F(z) \dif F(\tau)
	\end{split} 
	\end{align}
	We then extend the function $g_1(z)$ from the space $Z\rightarrow \mathbb{R}$ to the space $Z \times T \rightarrow \mathbb{R}$. We construct the function $g'_1(z,\tau) \in Z\times T \rightarrow \mathbb{R}$ such that $g'_1(z,\tau)=g_1(z)$ holds for all $\tau,z$. Then the traditional risk minimization can be expressed as
	\begin{align}
	\begin{split}
	& \inf_{g_1(z)\in G_1} R_t(g_1(z))  \\
	=&\inf_{g'_1(z)\in G'_1} \iint L_0(z,\tau,g'(z,\tau))\hat{h}(z,\tau) \dif F(z) \dif F(\tau)
	\end{split}
	\label{th0 proof eq2}
	\end{align}
	where $G'_1=\left \{g: g \in Z\times T \rightarrow \mathbb{R}, g(z,\tau)=\bar{g}(z)\  for\  \forall \tau  \right \}$.
	
	On the other hand, the global risk minimization in subjectivity learning is defined as
	\begin{align}
	\begin{split}
	&\inf_{g_2\in G_2,h\in H} R_s(g_2(z,\tau),h(z,\tau)) \\
	=&\inf_{g_2\in G_2,h\in H} \int L_0(z,\tau,g_2(z,\tau))h(z,\tau) \dif F(z) \dif F(\tau)
	\end{split}
	\end{align}
	where $G_2 = \{g: g \in Z\times T \rightarrow \mathbb{R}  \} $.
	When the data-subjects relation is determined by $\hat{h}(z,t)$, the global risk is formed as
	\begin{align}
	\begin{split}
	&\inf_{g_2\in G_2} R'_s(g_2(z,\tau)) \\
	=&\inf_{g_2\in G_2} \int L_0(z,\tau,g_2(z,\tau))\hat{h}(z,\tau) \dif F(z) \dif F(\tau)
	\label{th0 proof eq3}
	\end{split}
	\end{align}
	Since  $G'_1 \subseteq G_2$, compared with (\ref{th0 proof eq2}) and (\ref{th0 proof eq3}), we have that
	\begin{align}
	\inf_{g_1(z)\in G_1} R_t(g_1(z)) \geqslant \inf_{g_2\in G_2} R'_s(g_2(z,\tau))
	\end{align}
	take place for any division $\hat{h}(z,\tau)$.
	Also, $h(z,\tau)$ is a variable of the global risk functional that
	\begin{align}
	\inf_{g_2\in G_2} R'_s(g_2(z,\tau)) \geqslant \inf_{g_2\in G_2,h\in H} R_s(g_2(z,\tau),h(z,\tau)),
	\end{align}
	so we get the inequality
	\begin{align}
	\inf_{g_1(z)\in G_1} R_t(g_1(z)) \geqslant \inf_{g_2\in G_2,h\in H} R_s(g_2(z,t),h(z,t)).
	\end{align}
	The first  part  of  the theorem has been proved, which qualitatively shows the problem of global risk drives a lower risk bound.
	
	Then we consider a more realistic case. Let the optimal solution of global risk functional be $g^*(z,\tau)$ and $h^*(z,\tau)$. There should exist samples measured with $\dif F(z,\tau_1),\dif F(z,\tau_2)>0$ such that 
	\begin{align}
	g^*(z,\tau_1) \neq g^*(z,\tau_2),
	\end{align}
	which corresponds to the multi-label case of data in general intelligence problem. Generally, the loss function holds $L_0(z,\tau,g_1(z,\tau)) \neq L_0(z,\tau,g_2(z,\tau))$ when $g_1(z,\tau) \neq g_2(z,\tau)$.
	
	Under the optimal data-subject distribution $h^*(z,\tau)$, we expand the traditional risk functional by subject  $\tau$ that
	\begin{align*}
	\inf_{g_1} R_t(g_1)&=\inf_{g_1} \int L(z,g_1(z)) \dif F(z) \\
	&=\inf_{g_1} \int L(z,\tau,g_1(z))|t \cdot h^*(z,\tau) \dif F(z) \dif F(\tau) \\
	&=\inf_{g_1} \int L_0(z,\tau,g_1(z)) h^*(z,\tau) \dif F(z) \dif F(\tau)
	\end{align*}
	Let the optimal solution of traditional risk functional be $g_1^*(z)$. Then we have
	\begin{align}
	\inf_{g_1} R_t(g_1) = \int L_0(z,\tau,g_1^*(z)) h^*(z,\tau) \dif F(z) \dif F(\tau)
	\end{align}
	
	One the other hand, we consider the lower bound of global risk functional that
	\begin{align}
	\inf_{g,h} R_s(g,h)=\int L_0(z,\tau,g^*(z,\tau))h^*(z,\tau) \dif F(z) \dif F(\tau).
	\end{align}
	Since there exist samples $z,\tau_1$ and $z,\tau_2$ measured with $\dif F(z,\tau_1), \dif F(z,\tau_2)>0$ such that
	\begin{align}
	g^*(z,\tau_1) \neq g^*(z,\tau_2),
	\end{align}
	there must have
	\begin{align}
	g^*(z,\tau_1) \neq g_1^*(z) \ \ or\ \  g^*(z,\tau_2) \neq g_1^*(z).
	\end{align}
	Without generality, suppose that $g^*(z,\tau_1) \neq g_1^*(z)$. Since the loss function have
	\begin{align}
	L_0(z,\tau,g_1(z,\tau)) \neq L_0(z,\tau,g_2(z,\tau))
	\end{align}
	when $g_1(z,\tau) \neq g_2(z,\tau)$, we get
	\begin{align}
	L_0(z,\tau_1,g^*(z,\tau_1)) \neq L_0(z,\tau_1,g_1^*(z)).
	\end{align}
	And, because $g^*(z,\tau)$ is the optimal description for obtaining the lower bound, the inequality
	\begin{align}
	L_0(z,\tau,g^*(z,\tau)) \leq L_0(z,\tau,g_1^*(z))
	\end{align}
	holds for any $z,\tau$. If the inequality is not satisfied, obviously we can construct a new optimal solution $g^{**}$ such that $g^{**}(z,\tau)=g^*_1(z)$ on the interval where inequality dose not hold and $g^{**}(z,\tau)=g^*(z,\tau)$ on the other interval. 
	
	Therefore, we have
	\begin{align}
	L_0(z,\tau_1,g^*(z,\tau_1)) < L_0(z,\tau_1,g_1^*(z))
	\end{align} 
	for sample $z,\tau_1$ with $\dif F(z,\tau_1) > 0$.
	Then the risk gap between lower bound of traditional risk functional and global risk functional satisfies that:
	\begin{align*}
	\Delta R^* &= \inf_{g_1} R_t(g_1(z)) - \inf_{g,h} R_s(g,h) \\
	\begin{split}
	&=\int L_0(z,\tau,g^*_1(z)) h^*(z,\tau) \dif F(z) \dif F(t) \\
	&\qquad- \int L_0(z,\tau,g^*(z,\tau))h^*(z,\tau) \dif F(z) \dif F(\tau) 
	\end{split} \\
	\begin{split}
	&=\int [L_0(z,\tau,g^*_1(z)) \\
	&\qquad \quad-L_0(z,\tau,g^*(z,\tau))] h^*(z,\tau) \dif F(z) \dif F(\tau) 
	\end{split}\\
	&=\int [L_0(z,\tau,g^*_1(z))-L_0(z,\tau,g^*(z,\tau))] \dif F(z,\tau) \\
	&>0.
	\end{align*}
	The theorem is proved.
\end{proof}

\section{Theorem Proof in \\ Convergency with Two Coupled Variables}
For analysis the consistency of the principle of empirical global risk minimization, we first generalize the Law of Large Number to the case of two coupled number. We use the same notation of the global risk and the empirical risk function in the paper that
\begin{align}
R(\alpha) = \iint Q(z,\tau,\alpha) \dif F(z)\dif F(\tau)
\end{align}
and
\begin{align}
R_{emp}(\alpha, m, l) = \frac{1}{m} \sum_{j=1}^{m} {\frac{1}{l} \sum_{i=1}^{l} Q(z_i,\tau_j,\alpha)}.
\end{align}
We propose the theorem that
\begin{theorem}
	(\textbf{Convergency Theorem with two coupled variables.}) When the $\alpha^*$ is determined, for the function $Q(z,\tau,\alpha^*)$ and any $\varepsilon > 0$, the following convergence
	\begin{align}
	P\Big \{ R(\alpha^*)-R_{emp} (\alpha^*, m, l) >\varepsilon  \Big \} \xrightarrow[\left < m,l \right >\rightarrow \infty]{} 0
	\label{eq lemma 1}
	\end{align}
	take place, where the sample numbers $\left <m,l \right >$ satisfied the rule:
	\begin{align}
	l > \frac{2(B_z - A_z)^2}{\varepsilon ^2} \ln m + \frac{(B_z - A_z)^2}{(B_\tau - A_\tau)^2} m.
	\label{sample relation}
	\end{align}
	The $B_z,A_z$ and $B_\tau, A_\tau$ are respectively the bound of functions that $A_z \leq Q(z,\tau) \leq B_z$ and $A_\tau \leq R^{lo}(\alpha, \tau) \leq B_\tau$.
	\label{lemma 1}
\end{theorem}

Before the proof, we first introduce two basic inequality of probability. The first one is
\begin{lemma}
	The inequality
	\begin{align}
	P\left \{x_1+x_2 > \varepsilon\right \} \leq P\left \{x_1>\frac{\varepsilon}{2}\right \} + P\left \{x_2>\frac{\varepsilon}{2}\right \}
	\end{align}
	holds true.
	\label{lemma111}
\end{lemma}
\begin{proof}
	\begin{align*}
	&P\left \{x_1+x_2 > \varepsilon \right \} \\
	\begin{split}
	= & 	P\left \{x_1+x_2 > \varepsilon \mid x_1 > \frac{\varepsilon}{2} \right \}P\left \{ x_1 > \frac{\varepsilon}{2} \right \}   \\
	& + P\left \{x_1+x_2 > \varepsilon \mid x_1 \leq \frac{\varepsilon}{2} \right \}P\left \{ x_1 \leq \frac{\varepsilon}{2} \right \}
	\end{split} \\
	\begin{split}
	\leq & 	P\left \{x_1+x_2 > \varepsilon \mid x_1 > \frac{\varepsilon}{2} \right \}P\left \{ x_1 > \frac{\varepsilon}{2} \right \}   \\
	& + P\left \{ x_2 > \frac{\varepsilon}{2} \mid x_1 \leq \frac{\varepsilon}{2} \right \}P\left \{ x_1 \leq \frac{\varepsilon}{2} \right \}
	\end{split} \\
	\leq &   P\left \{ x_1 > \frac{\varepsilon}{2} \right \} + P\left \{ x_2 > \frac{\varepsilon}{2}  \right \}
	\end{align*}
\end{proof}
The second lemma is the generalization of Lemma \ref{lemma111} for multi-variables that
\begin{lemma}
	The inequality
	\begin{align}
	P\left \{\frac{1}{n} \sum_{i=1}^{n} x_i > \varepsilon\right \} \leq \sum_{i=1}^{n}  P\left \{x_i> \varepsilon \right \} 
	\end{align}
	holds true.
	\label{lemma222}
\end{lemma}
\begin{proof}
	Let us prove it by mathematical induction. For $n=1$, the inequality obviously holds true.
	For $n>1$, we assume that the inequality holds true for $n-1$ that
	\begin{align}
	P\left \{\frac{1}{n-1} \sum_{i=1}^{n-1} x_i > \varepsilon\right \} \leq \sum_{i=1}^{n-1}  P\left \{x_i> \varepsilon \right \}. 
	\end{align}
	Then we have
	\begin{align*}
	&P\left \{\frac{1}{n} \sum_{i=1}^{n} x_i > \varepsilon\right \}  \\
	=& P\left \{ x_n + \sum_{i=1}^{n-1} x_i > n \cdot \varepsilon \right \} \\
	\begin{split}
	= & 	P\left \{x_n + \sum_{i=1}^{n-1} x_i > n\varepsilon \mid x_n > \varepsilon \right \}P\left \{ x_n > \varepsilon \right \}   \\
	& + P\left \{x_n + \sum_{i=1}^{n-1} x_i > n\varepsilon \mid x_n \leq \varepsilon \right \}P\left \{ x_n \leq \varepsilon \right \}
	\end{split} \\
	\begin{split}
	\leq & P\left \{ x_n > \varepsilon \right \}P\left \{x_n + \sum_{i=1}^{n-1} x_i > n\varepsilon \mid x_n > \varepsilon \right \}	   \\
	& + P\left \{ \sum_{i=1}^{n-1} x_i > (n-1)\varepsilon \mid x_n \leq \varepsilon \right \}P\left \{ x_n \leq \varepsilon \right \}
	\end{split} \\
	\leq & P\left \{ x_n > \varepsilon \right \} + P\left \{ \sum_{i=1}^{n-1} x_i > (n-1)\varepsilon \right \}     \\
	\leq & \sum_{i=1}^{n}  P\left \{x_i> \varepsilon \right \} 
	\end{align*}
	Therefore, the lemma is proved.
\end{proof}

Now we prove the Theorem \ref{lemma 1}.
\begin{proof}[\textbf{Proof of Theorem \ref{lemma 1}}:]
	To prove the theorem, we rewrite the equation (\ref{eq lemma 1}) as that: For any $\varepsilon > 0, \epsilon > 0$, we can find a $\delta_1 > 0$ so that when $m^* > \delta_1$ and $l^* > \delta_2(m^*)$, the following inequality
	\begin{align}
	P\Big \{ R(\alpha^*)-R_{emp} (\alpha^*, m^*, l^*) >\varepsilon  \Big \} < \epsilon ,
	\label{lemma1 proof 1}
	\end{align}
	take place, where the local risk is defined as
	\begin{align}
	R^{lo} (\tau,\alpha)= \int Q(z,\tau,\alpha) \dif F(z). 
	\end{align}
	Then we consider left side of equation(\ref{lemma1 proof 1}) as 
	\begin{align}
	&P\Big\{   R(\alpha^*)-R_{emp} (\alpha^*, m^*, l^*)  >\varepsilon  \Big\} \\
	\begin{split}
	=& P\Big\{  R(\alpha^*) - \frac{1}{m} \sum_{j=1}^{m} R^{lo} (\tau_j,\alpha^*) \\
	&+ \frac{1}{m} \sum_{j=1}^{m} R^{lo} (\tau_j,\alpha^*) -  R_{emp} (\alpha^*, m^*, l^*) >\varepsilon  \Big\}
	\end{split} 
	\end{align}
	\begin{align}
	\begin{split}
	\ \  \leq & P\Big\{  R(\alpha^*) - \frac{1}{m} \sum_{j=1}^{m} R^{lo} (\tau_j,\alpha^*) > \frac{\varepsilon}{2} \Big\} \\
	&+P\Big\{  \frac{1}{m} \sum_{j=1}^{m} R^{lo} (\tau_j,\alpha^*)  \\
	&\qquad  \qquad \qquad \qquad -R_{emp} (\alpha^*, m^*, l^*)  >\frac{\varepsilon}{2}  \Big\}
	\label{lemma1 proof ineq 1}
	\end{split} 
	\end{align}
	\begin{align}
	\begin{split}
	\ \ = & P\Big\{   R(\alpha^*) - \frac{1}{m} \sum_{j=1}^{m} R^{lo} (\tau_j,\alpha^*) > \frac{\varepsilon}{2} \Big\} \\
	&+P\Big\{  \frac{1}{m} \sum_{j=1}^{m} \big ( R^{lo} (\tau_j,\alpha^*) \\
	&\qquad  \qquad \qquad - \frac{1}{l} \sum_{i=1}^{l} Q(z_i,\tau_j,\alpha^*) \big )  >\frac{\varepsilon}{2}  \Big\}
	\end{split} 
	\end{align}
	\begin{align}
	\begin{split}
	\ \ \leq & P\Big\{  R(\alpha^*) - \frac{1}{m} \sum_{j=1}^{m} R^{lo} (\tau_j,\alpha^*)  > \frac{\varepsilon}{2} \Big\} \\
	&+  P\Big\{ \frac{1}{m} \sum_{j=1}^{m} \Big [  R^{lo} (\tau_j,\alpha^*) \\
	&\qquad  \qquad \qquad  - \frac{1}{l} \sum_{i=1}^{l} Q(z_i,\tau_j,\alpha^*)  \Big ] >\frac{\varepsilon}{2}  \Big\}
	\label{lemma1 proof ineq 2}
	\end{split}
	\end{align}
	\begin{align}
	\begin{split}
	\ \ \leq & P\Big\{  R(\alpha^*) - \frac{1}{m} \sum_{j=1}^{m} R^{lo} (\tau_j,\alpha^*)  > \frac{\varepsilon}{2} \Big\} \\
	&+ \sum_{j=1}^{m} P\Big\{  R^{lo} (\tau_j,\alpha^*) \\
	&\qquad  \qquad \qquad - \frac{1}{l} \sum_{i=1}^{l} Q(z_i,\tau_j,\alpha^*)  >\frac{\varepsilon}{2}  \Big\}
	\label{lemma1 proof ineq 2}
	\end{split}
	\end{align}
	Note that the inequality (\ref{lemma1 proof ineq 1}) and (\ref{lemma1 proof ineq 2}) utilize the Lemma \ref{lemma111} and Lemma \ref{lemma222} respectively.
	
	From the Law of Large Numbers, we have
	\begin{align}
	\small
	\begin{split}
	P\left\{  \int Q(z,\alpha^*) \dif F(z) - \frac{1}{l} \sum_{i=1}^{l} Q(z_i,\alpha^*)  > \varepsilon \right\} 	\xrightarrow[l\rightarrow \infty]{} 0,
	\end{split}	
	\end{align}
	that is, for any $\epsilon > 0$ we can find a $\delta$ such that when $l>\delta$ the inequality
	\begin{align}
	P\Big\{  \int Q(z,\alpha^*) \dif F(z) - \frac{1}{l} \sum_{i=1}^{l} Q(z_i,\alpha^*)  > \varepsilon \Big\} < \epsilon
	\end{align}
	takes place.
	
	Therefore, we can find a $\delta_1>0$ so that when $m^* > \delta_1$ we have
	\begin{align}
	&P\Big\{  R(\alpha^*) - \frac{1}{m^*}\sum_{j=1}^{m^*} R^{lo} (\tau_j,\alpha^*)  > \frac{\varepsilon}{2} \Big\} \label{lemma1 proof ineq 3_1}\\ 
	\begin{split}
	=&P\Big\{  \int R^{lo} (\tau,\alpha^*) \dif F(\tau)  - \frac{1}{m^*}\sum_{j=1}^{m^*} R^{lo} (\tau_j,\alpha^*) > \frac{\varepsilon}{2}  \Big\} < \frac{\epsilon}{2}
	\end{split}
	\label{lemma1 proof ineq 3}
	\end{align}
	When $m^*$ is determined, we can find a $\delta_2(m^*)$ such that when $l^*>\delta_2(m^*)$ we have
	\begin{align}
	P\Big\{  R^{lo} (\tau_j,\alpha^*) - \frac{1}{l^*}\sum_{i=1}^{l^*} Q(z_i,\tau_j,\alpha^*) >\frac{\varepsilon}{2}  \Big\} < \frac{\epsilon}{2m}
	\label{lemma1 proof ineq 4}.		
	\end{align}
	Now we need to find the form of relation $l > \delta_2(m^*)$. 
	
	From the Hoeffding's inequality, we give the probability form of (\ref{lemma1 proof ineq 3_1}) and (\ref{lemma1 proof ineq 4}) that:
	\begin{align}
	\begin{split}
	P\Big\{  R(\alpha^*) - \frac{1}{m^*}&\sum_{j=1}^{m^*} R^{lo} (\tau_j,\alpha^*)  > \frac{\varepsilon}{2} \Big\} \\
	&\leq  exp\Bigg ( -\frac{m^*\varepsilon^2}{2(B_\tau-A_\tau)^2}    \Bigg )
	\end{split}	
	\end{align}
	and
	\begin{align}
	\begin{split}
	P\Big\{ R^{lo} (\tau_j,\alpha^*) - \frac{1}{l^*}&\sum_{i=1}^{l^*} Q(z_i,\tau_j,\alpha^*) >\frac{\varepsilon}{2}  \Big\} \\
	&\leq  exp\Bigg ( -\frac{l^*\varepsilon^2}{2(B_z-A_z)^2}    \Bigg ).
	\end{split}	
	\end{align}
	Since the $m^* > \delta_1$ makes equation ($lemma1 proof ineq 3$) be true, we assume that
	\begin{align}
	exp\Bigg ( -\frac{m^*\varepsilon^2}{2(B_\tau-A_\tau)^2} \Bigg ) < \frac{\epsilon}{2}.
	\end{align}
	To make the inquality (\ref{lemma1 proof ineq 4}) take place, we can establish a sufficient condition that:
	\begin{align}
	exp\Bigg ( -\frac{l^*\varepsilon^2}{2(B_z-A_z)^2}    \Bigg ) < \frac{1}{m} exp\Bigg ( -\frac{m^*\varepsilon^2}{2(B_\tau-A_\tau)^2} \Bigg )
	\end{align}
	Rewrite the form of above inequality, we have
	\begin{align}
	l > \frac{2(B_z - A_z)^2}{\varepsilon ^2} \ln m + \frac{(B_z - A_z)^2}{(B_\tau - A_\tau)^2} m.
	\end{align}
	With this relation, when $m^* > \delta_1$, the inequalities (\ref{lemma1 proof ineq 3_1}) and (\ref{lemma1 proof ineq 4}) all take place.
	
	Bringing the inequality (\ref{lemma1 proof ineq 3}) and (\ref{lemma1 proof ineq 4}) into inequation (\ref{lemma1 proof ineq 2}), we have
	\begin{align}
	&P\Big\{  R(\alpha^*)-R_{emp} (\alpha^*, m^*, l^*) >\varepsilon  \Big\} < \epsilon.
	\end{align}	
	In summary, we get 
	\begin{align}
	P\Big \{  R(\alpha^*)-R_{emp} (\alpha^*, m, l)  >\varepsilon  \Big \} \xrightarrow[ \left <m,l \right >\rightarrow \infty]{} 0.
	\end{align}
\end{proof}

\section{Proof of The Equivalent Theorem}

When we consider the condition of consistency, we transfer the problem of consistency to the one-sided uniform convergence. The equivalent theorem demonstrates that one-sided uniform convergency
\begin{align}
P\left\{  \sup _{\alpha \in \Lambda}  \left(  R(\alpha)-R_{emp} (\alpha, m, l) \right) > \varepsilon \right\} \xrightarrow[\left < m,l \right >\rightarrow \infty]{} 0
\label{one-sided convergency}
\end{align}
forms not only the sufficient conditions for the consistency of the EGRM, but the necessary conditions as well. The definition of consistency is that :

\textbf{Definition of Consistency}
We say that the method of global empirical risk minimization is strictly (nontrivially) consistent the set of function $Q(z,\tau,\alpha),\alpha \in \Lambda$ if for any nonempty subset $\Lambda(c), c\in (-\infty, \infty)$ of this set of functions such that
\begin{align}
\Lambda(c) = \{\alpha : \iint Q(z,\tau,\alpha) \dif F(z)\dif F(\tau) \geq c \},
\end{align}
the next convergence is valid:
\begin{align}
\inf_{\alpha \in \Lambda(c)} R_{emp}(\alpha)  \xrightarrow[\left < m,l \right >\rightarrow \infty]{P}   \inf_{\alpha \in \Lambda(c)} R(\alpha)
\label{eq consistency}
\end{align}

Now we prove the equivalent theorem:
\begin{theorem}
	\textbf{(the Equivalent Theorem)} Let there exist the constants $a$ and $A$ such that for all functions in the set $Q(z,\tau,\alpha),\alpha \in \Lambda$ and for distribution functions $F(t)$ and $F(z)$, the inequalities
	\begin{align}
	a \leq \iint Q(z,\tau,\alpha) \dif F(z)\dif F(\tau) \leq A
	\end{align}
	hold true. Then the following two statements are equivalent:
	
	1. The empirical global risk minimization method is strictly consistent (\ref{eq consistency}) on the set of functions $Q(z,\tau,\alpha)$.
	
	2. The uniform one-sided convergence of the means to their mathematical expectation (\ref{one-sided convergency}) takes place over the set of functions $Q(z,\tau,\alpha)$.
\end{theorem}

\begin{proof}
	Let the global empirical risk minimization method be strictly consistent on the set of functions $Q(z,\tau, \alpha)$. According to the definition of strictly consistency, this means that for $c$ such that the set
	\begin{align}
	\Lambda(c) = \{\alpha : \iint Q(z,t,\alpha) \dif F(z)\dif F(t) \geq c \}
	\end{align}
	is noempty the following convergence in probability is true:
	\begin{align}
	\inf_{\alpha \in \Lambda(c)} R_{emp}(\alpha,m,l)  \xrightarrow[ \left <m,l \right >\rightarrow \infty]{P}   \inf_{\alpha \in \Lambda(c)} R(\alpha)
	\label{theorem1 proof eq1}
	\end{align}
	
	Consider a finite sequence of numbers $a_1,...,a_n$ such that
	\begin{align}
	|a_{i+1}-a_i|<\frac{\varepsilon}{2}, \qquad a_1=a,a_n=A
	\end{align}
	We denote by $G_k$ the event
	\begin{align}
	\inf_{\alpha \in \Lambda(a_k)} R_{emp}(\alpha,m,l) < \inf_{\alpha \in \Lambda(a_k)} R(\alpha) - \frac{\varepsilon}{2},
	\end{align}
	that is
	\begin{align}
	\begin{split}
	&\inf_{\alpha \in \Lambda(a_k)} \frac{1}{m} \sum_{j=1}^{m} {\frac{1}{l} \sum_{i=1}^{l} Q(z_i,\tau_j,\alpha)} \\
	&\qquad < \inf_{\alpha \in \Lambda(a_k)} \iint Q(z,\tau,\alpha) \dif F(z)\dif F(\tau) - \frac{\varepsilon}{2}
	\end{split}	
	\end{align}
	By the consistency of (\ref{theorem1 proof eq1}), we have
	\begin{align}
	P(G_k) \xrightarrow[ \left< m,l \right> \rightarrow \infty]{P} 0.
	\end{align} 
	We denote
	\begin{align}
	G=\bigcup _{k=1}^{n} G_k.
	\end{align}
	Since $n$ is finite and for any $k$ the equation (\ref{theorem1 proof eq1}) is true, it follows that
	\begin{align}
	P(G) \xrightarrow[ \left < m,l \right > \rightarrow \infty]{P} 0
	\label{theorem1 proof eq2}
	\end{align}
	We denote by $\mathcal{A}$ the event
	\begin{align}
	\begin{split}
	\sup_{\alpha \in \Lambda} \Big ( \iint Q(z,\tau,&\alpha) \dif F(z)\dif F(\tau)  \\
	-& \frac{1}{m} \sum_{j=1}^{m} {\frac{1}{l} \sum_{i=1}^{l} Q(z_i,\tau_j,\alpha)} \Big ) > \varepsilon
	\end{split}
	\end{align}
	Then we compare the event $\mathcal{A}$ and the event $G$. Suppose that $\mathcal{A}$ takes place, then we can find an $\alpha^* \in \Lambda$ such that
	\begin{align}
	\begin{split}
	\iint Q(z,\tau,\alpha^*) &\dif F(z)\dif F(\tau) - \varepsilon \\
	&> \frac{1}{m} \sum_{j=1}^{m} {\frac{1}{l} \sum_{i=1}^{l} Q(z_i,\tau_j,\alpha^*)}
	\end{split}
	\end{align}
	From $\alpha^*$ we find $k$ such that $\alpha^* \in \Lambda(a_k)$ and
	\begin{align}
	\begin{split}
	\iint Q(z,\tau,\alpha^*) &\dif F(z)\dif F(\tau) - a_k < \frac{\varepsilon}{2}
	\end{split}
	\end{align}
	For the chosen set $\Lambda(a_k)$, the inequality
	\begin{align}
	\begin{split}
	\iint Q&(z,\tau,\alpha^*) \dif F(z)\dif F(\tau) \\
	&-\inf_{\alpha \in \Lambda(a_k)} \iint Q(z,\tau,\alpha) \dif F(z)\dif F(\tau) < \frac{\varepsilon}{2}
	\end{split}
	\end{align}
	holds true.
	Therefore for the chosen $\alpha^*$ and the set $\Lambda(a_k)$, then the following inequalities take place:
	\begin{align}
	&\inf_{\alpha \in \Lambda(a_k)} \iint Q(z,\tau,\alpha) \dif F(z)\dif F(\tau) - \frac{\varepsilon}{2} \\
	>& \iint Q(z,\tau,\alpha^*) \dif F(z)\dif F(\tau) -\varepsilon \\
	>& \frac{1}{m} \sum_{j=1}^{m} {\frac{1}{l} \sum_{i=1}^{l} Q(z_i,\tau_j,\alpha^*)} \\
	\geq & \inf_{\alpha \in \Lambda(a_k)} \frac{1}{m} \sum_{j=1}^{m} {\frac{1}{l} \sum_{i=1}^{l} Q(z_i,\tau_j,\alpha)},
	\end{align}
	that is, the event $G_k$ does occur and, hence, so does $G$.
	From above derivation, we have
	\begin{align}
	P(\mathcal{A}) < P(G).
	\end{align}
	By equation(\ref{theorem1 proof eq2}),
	\begin{align}
	\lim_{\left <l,m \right > \rightarrow \infty} P(G)=0,
	\end{align}
	which expresses uniform one-sided convergence
	\begin{align}
	\begin{split}
	P\Big \{  & \sup_{\alpha \in \Lambda}  \big ( \iint Q(z,\tau,\alpha) \dif F(z)\dif F(\tau) \\
	&- \frac{1}{m} \sum_{j=1}^{m} {\frac{1}{l} \sum_{i=1}^{l} Q(z_i,\tau_j,\alpha)}  \big )  \Big \} \xrightarrow[\left < m,l \right > \rightarrow \infty]{P} 0
	\label{theorem1 proof eq3}
	\end{split}
	\end{align}
	So far, the first part of the theorem is proved.
	Now suppose that uniform one-sided convergence (\ref{theorem1 proof eq3}) takes place. We need to prove that the strict consistency takes place in this case. It is for any $\varepsilon$ the convergence
	\begin{align}
	\begin{split}
	\lim _{l\rightarrow \infty} & P \Big \{ \Big | \inf_{\alpha \in \Lambda(c)} \iint Q(z,\tau,\alpha) \dif F(z)\dif F(\tau) \\
	&-  \inf_{\alpha \in \Lambda(c)} \frac{1}{m} \sum_{j=1}^{m} {\frac{1}{l} \sum_{i=1}^{l} Q(z_i,\tau_j,\alpha)}  \Big | > \varepsilon \Big \} =0
	\end{split}
	\end{align}
	holds. Let us denote by $\mathcal{A}$ the event
	\begin{align}
	\begin{split}
	\Big | \inf_{\alpha \in \Lambda(c)} & \iint Q(z,\tau,\alpha) \dif F(z)\dif F(\tau) \\
	&-  \inf_{\alpha \in \Lambda(c)} \frac{1}{m} \sum_{j=1}^{m} {\frac{1}{l} \sum_{i=1}^{l} Q(z_i,\tau_j,\alpha)} \Big | > \varepsilon.
	\end{split}
	\end{align}
	Then the event $\mathcal{A}$ is the union of two ond-sided events
	\begin{align}
	\mathcal{A} = \mathcal{A}_1 \bigcup \mathcal{A}_2,
	\end{align}
	where
	\begin{align}
	\begin{split}
	\mathcal{A}_1 = \Big \{ \inf_{\alpha \in \Lambda(c)} & \iint Q(z,\tau,\alpha) \dif F(z)\dif F(\tau) + \varepsilon \\
	<&  \inf_{\alpha \in \Lambda(c)} \frac{1}{m} \sum_{j=1}^{m} {\frac{1}{l} \sum_{i=1}^{l} Q(z_i,\tau_j,\alpha)}  \Big \}
	\end{split}
	\end{align}
	and
	\begin{align}
	\begin{split}
	\mathcal{A}_2 = \Big \{ \inf_{\alpha \in \Lambda(c)} & \iint Q(z,\tau,\alpha) \dif F(z)\dif F(\tau) - \varepsilon \\
	>&  \inf_{\alpha \in \Lambda(c)} \frac{1}{m} \sum_{j=1}^{m} {\frac{1}{l} \sum_{i=1}^{l} Q(z_i,\tau_j,\alpha)}  \Big \}.
	\end{split}
	\end{align}
	Then we bound the probability of the event $\mathcal{A}$
	\begin{align}
	P(\mathcal{A}) \leq P(\mathcal{A}_1) + P(\mathcal{A}_2).
	\end{align}
	Suppose that the event $\mathcal{A}_1$ occurs. To bound $P(\mathcal{A}_1)$ we take a function $Q(z,\tau,\alpha^*)$ such that
	\begin{align}
	\begin{split}
	\iint Q(z,\tau,&\alpha^*) \dif F(z)\dif F(\tau) \\
	< & \inf_{\alpha \in \Lambda(c)}  \iint Q(z,\tau,\alpha) \dif F(z)\dif F(\tau) + \frac{\varepsilon}{2}.
	\end{split}
	\end{align}
	Then the inequality
	\begin{align}
	\begin{split}
	\frac{1}{m} \sum_{j=1}^{m} & \frac{1}{l}  \sum_{i=1}^{l} Q(z_i,\tau_j,\alpha^*) \\
	& > \iint Q(z,\tau,\alpha^*) \dif F(z)\dif F(\tau) + \frac{\varepsilon}{2}
	\end{split}
	\end{align}
	holds. The probability of this inequality is therefore not less that the probability of the event $\mathcal{A}_1$:
	\begin{align}
	\begin{split}
	& P(\mathcal{A}_1) \\
	&\leq P \Big \{ \frac{1}{m} \sum_{j=1}^{m} \frac{1}{l}  \sum_{i=1}^{l} Q(z_i,\tau_j,\alpha^*) \\
	& \qquad - \iint Q(z,\tau,\alpha^*) \dif F(z)\dif F(\tau) > \frac{\varepsilon}{2} \Big \} \\
	\end{split}
	\end{align}
	The probability on the right-hand side tends to zero by the generation of the law of large numbers (Theorem \ref{lemma 1}), that is
	\begin{align}
	\begin{split}
	P & \Big \{ \frac{1}{m} \sum_{j=1}^{m} \frac{1}{l}  \sum_{i=1}^{l} Q(z_i,\tau_j,\alpha^*) \\
	- & \iint Q(z,\tau,\alpha^*) \dif F(z)\dif F(\tau) > \frac{\varepsilon}{2} \Big \} \xrightarrow[ \left <m,l \right >\rightarrow \infty]{P} 0.
	\end{split}
	\end{align}
	Therefore, we conclude that
	\begin{align}
	P(\mathcal{A}_1) \xrightarrow[\left <m,l \right  >\rightarrow \infty]{P} 0.
	\label{theorem1 proof eq4}
	\end{align}
	On the other hand, the event $\mathcal{A}_2$ occurs, then there is a function $Q(z,\tau,\alpha^{**}), \alpha^{**} \in \Lambda(c)$ such that
	\begin{align}
	\begin{split}
	&\frac{1}{m} \sum_{j=1}^{m} \frac{1}{l}  \sum_{i=1}^{l} Q(z_i,\tau_j,\alpha^{**}) + \frac{\varepsilon}{2} \\
	< &  \inf_{\alpha \in \Lambda(c)}  \iint Q(z,\tau,\alpha) \dif F(z)\dif F(\tau) \\
	< & \iint Q(z,\tau,\alpha^{**}) \dif F(z)\dif F(\tau).
	\end{split}
	\end{align}
	Therefore, the relation
	\begin{align}
	\begin{split}
	& P(\mathcal{A}_2) \\
	< & P \Big \{ \iint Q(z,\tau,\alpha^{**}) \dif F(z)\dif F(\tau) \\
	& \qquad \qquad - \frac{1}{m} \sum_{j=1}^{m} \frac{1}{l}  \sum_{i=1}^{l} Q(z_i,\tau_j,\alpha^{**}) > \frac{\varepsilon}{2}   \Big \} \\
	< & P \Big \{  	\sup_{\alpha \in \Lambda} \Big ( \iint Q(z,\tau,\alpha) \dif F(z)\dif F(\tau)  \\
	& \qquad \qquad - \frac{1}{m} \sum_{j=1}^{m} {\frac{1}{l} \sum_{i=1}^{l} Q(z_i,\tau_j,\alpha)} \Big ) > \varepsilon  \Big \} \\
	&\qquad \qquad \qquad \qquad \qquad \qquad \qquad \quad  \xrightarrow[ \left <m,l \right >\rightarrow \infty]{P} 0
	\end{split}
	\label{theorem1 proof eq5}
	\end{align}
	holds by virtue of (\ref{theorem1 proof eq3}).
	Since
	\begin{align}
	P(\mathcal{A}) \leq P(\mathcal{A}_1) + P(\mathcal{A}_2),
	\end{align}
	from equation (\ref{theorem1 proof eq4}) and (\ref{theorem1 proof eq5}) we conclude that
	\begin{align}
	P(\mathcal{A}) \xrightarrow[ \left <m,l \right >\rightarrow \infty]{P} 0.
	\end{align}
	The theorem is proven.
\end{proof}

\section{Proof of Consistency Condition}

With the Equivalent Theorem, we should consider the conditions for uniform convergence (\ref{one-sided convergency}). We also use the local risk 
\begin{align}
R^{lo}(\alpha, \tau) = \int Q(z,\tau,\alpha)dF(z)	
\end{align}
for the subject $\tau$. Then, the following inequalities is valid:
\begin{theorem}
	For any $\varepsilon>0$, the following inequality holds:
	\begin{align}
	\begin{split}
	&P\Big \{  \sup _{\alpha \in \Lambda}  \big (  R(\alpha)-R_{emp} (\alpha, m, l) \big ) > \varepsilon \Big \} \\
	\leq & P\Big\{  \sup _{\alpha \in \Lambda}  \big (  \int R^{lo}(\alpha,\tau)\dif F(\tau) -\frac{1}{m}\sum_{j=1}^{m} R^{lo}(\tau_j,\alpha) \big ) > \varepsilon \Big \} \\
	& + \sum_{j=1}^{m} P\Big\{  \sup _{\alpha \in \Lambda}  \big (  \int Q(z,\tau_j,\alpha)\dif F(\tau) \\
	&\qquad \qquad\qquad\qquad- \frac{1}{l}\sum_{i=1}^{l} Q(z_i,\tau_j,\alpha) \big ) > \varepsilon \Big \}
	\end{split}
	\label{pro eq1}
	\end{align}
\end{theorem}

\begin{proof}
	\begin{flalign*}
	\  &P\Big \{  \sup _{\alpha \in \Lambda}  \big (  R(\alpha)-R_{emp} (\alpha, m, l) \big ) > \varepsilon \Big \} \\
	\begin{split}
	=&P\Big \{  \sup _{\alpha \in \Lambda}  \big (  R(\alpha)- \frac{1}{m} \sum_{j=1}^{m} R^{lo}(\tau_j,\alpha) \\
	& \qquad + \frac{1}{m} \sum_{j=1}^{m} R^{lo}(\tau_j,\alpha) -R_{emp} (\alpha, m, l)  \big ) > \varepsilon \Big \}
	\end{split}
	\end{flalign*}
	\begin{flalign}
	\begin{split}
	\leq&P\Big \{  \sup _{\alpha \in \Lambda}  \big (  R(\alpha)- \frac{1}{m} \sum_{j=1}^{m} R^{lo}(\tau_j,\alpha)\big ) \\
	& \  + \sup _{\alpha \in \Lambda}  \big ( \frac{1}{m} \sum_{j=1}^{m} R^{lo}(\tau_j,\alpha) -R_{emp} (\alpha, m, l)  \big ) > \varepsilon \Big \}
	\end{split}
	\end{flalign}
	\begin{flalign}
	\begin{split}
	\leq&P\Big \{  \sup _{\alpha \in \Lambda}  \big (  R(\alpha)- \frac{1}{m} \sum_{j=1}^{m} R^{lo}(\tau_j,\alpha)\big ) > \frac{\varepsilon}{2} \Big \} \\
	& \  +P\Big \{ \sup _{\alpha \in \Lambda}  \big ( \frac{1}{m} \sum_{j=1}^{m} R^{lo}(\tau_j,\alpha)  \\
	& \qquad \qquad \quad- \frac{1}{m} \sum_{j=1}^{m} \frac{1}{l} \sum_{i=1}^{l}Q(z_i,\tau_j,\alpha)  \big ) > \frac{\varepsilon}{2} \Big \}
	\end{split}
	\label{th4 eq1}
	\end{flalign}
	\begin{flalign}
	\begin{split}
	\leq&P\Big \{  \sup _{\alpha \in \Lambda}  \big (  R(\alpha)- \frac{1}{m} \sum_{j=1}^{m} R^{lo}(\tau_j,\alpha)\big ) > \frac{\varepsilon}{2} \Big \} \\
	& \  +P\Big \{ \frac{1}{m} \sum_{j=1}^{m} \sup _{\alpha \in \Lambda}  \big (  R^{lo}(\tau_j,\alpha)  \\
	& \qquad \qquad \qquad-  \frac{1}{l} \sum_{i=1}^{l}Q(z_i,\tau_j,\alpha)  \big ) > \frac{\varepsilon}{2} \Big \}
	\end{split}		
	\end{flalign}
	\begin{flalign}
	\begin{split}
	\leq&P\Big \{  \sup _{\alpha \in \Lambda}  \big (  R(\alpha)- \frac{1}{m} \sum_{j=1}^{m} R^{lo}(\tau_j,\alpha)\big ) > \frac{\varepsilon}{2} \Big \} \\
	& \  + \sum_{j=1}^{m} P\Big \{  \sup _{\alpha \in \Lambda}  \big (  R^{lo}(\tau_j,\alpha)  \\
	& \qquad \qquad \qquad-  \frac{1}{l} \sum_{i=1}^{l}Q(z_i,\tau_j,\alpha)  \big ) > \frac{\varepsilon}{2} \Big \}
	\end{split}
	\label{th4 eq2}		
	\end{flalign}
	\begin{flalign}
	\begin{split}
	=&P\Big \{  \sup _{\alpha \in \Lambda}  \big (  \int R^{lo}(\alpha,\tau)\dif F(\tau)- \frac{1}{m} \sum_{j=1}^{m} R^{lo}(\tau_j,\alpha)\big ) > \frac{\varepsilon}{2} \Big \} \\
	& \  + \sum_{j=1}^{m} P\Big \{  \sup _{\alpha \in \Lambda}  \big (  \int Q(z,\tau_j,\alpha)\dif F(\tau)  \\
	& \qquad \qquad \qquad \qquad-  \frac{1}{l} \sum_{i=1}^{l}Q(z_i,\tau_j,\alpha)  \big ) > \frac{\varepsilon}{2} \Big \}.
	\end{split}		
	\end{flalign}
	In the proof, the inequality (\ref{th4 eq1}) and inequality (\ref{th4 eq2}) respectively use the Lemma \ref{lemma111} and Lemma \ref{lemma222}.
\end{proof}

The convergence probability consists of two terms, where the first term is the convergence probability of the observations risk and the second term is the sum of convergence probability of samples risk under the specific observations. We further use the concept of capacity  to discuss the conditions of uniform convergence. Due to space limitations, the representation in the main text is brief, here we present it in detail.

Let us consider the first term
\begin{align}
P\Big\{  \sup _{\alpha \in \Lambda}  \big (  \int R^{lo}(\alpha,\tau)\dif F(\tau) -\frac{1}{m}\sum_{j=1}^{m} R^{lo}(\tau_j,\alpha) \big ) > \varepsilon \Big \}.
\end{align}
Let $R^{lo}(\tau,\alpha),\tau \in T, \alpha \in \Lambda$ be a set of real-valued functions. Let $N_t^{\Lambda,\beta_t}(\tau_1,...,\tau_m)$ be the number of different separations of $m$ vectors $\tau_1,...,\tau_m$ by a complete set of indicators:
\begin{align*}
&\theta \{ R^{lo}(\tau,\alpha) - \beta_\tau \},\\
&\ \alpha\in\Lambda,\ \beta_\tau\in \mathcal{B}_\tau=\Big ( \inf_{\alpha,\tau}R^{lo}(\tau,\alpha) \leq \beta_\tau \leq  \sup_{\alpha,\tau}R^{lo}(\tau,\alpha)  \Big ).
\end{align*}
Then we define the annealed entropy of subjectivity risk that

\begin{definition}
	(\textbf{Annealed Entropy of Subjectivity Risk}) 
	Let the function
	\begin{align}
	H_\tau^{\Lambda,\beta_\tau}(\tau_1,...,\tau_m) = \ln N_t^{\Lambda,\beta_\tau}(\tau_1,...,\tau_m)
	\end{align}
	be measurable with respect to measure on $\tau_1,...,\tau_m$.
	The quantity
	\begin{align}
	\hat{H}_{\tau}^{\Lambda,\beta_t}(m) = \ln E N_\tau^{\Lambda,\beta_\tau}(\tau_1,...,\tau_m)
	\end{align}
	is defined as the annealed entropy of the set indicators $\theta \{ R^{lo}(\tau,\alpha) - \beta_\tau \}$ of real-valued functions $R^{lo}(\tau,\alpha)$ .
\end{definition}

Using the error equality in statistical learning theory, for the bounded real-valued functions $A_t \leq R^{lo}(t,\alpha)\leq B_t,\alpha \in \Lambda$, the following inequality is valid:
\begin{align}
\begin{split}
P\Big\{  \sup _{\alpha \in \Lambda}  \big (  \int & R^{lo}(\alpha,\tau) \dif F(\tau) -\frac{1}{m}\sum_{j=1}^{m} R^{lo}(\tau_j,\alpha) \big ) > \varepsilon \Big \} \\
& \leq 4exp\Big \{ \Big ( \frac{\hat{H}_{\tau}^{\Lambda,\beta_\tau}(2m)}{m} - \frac{(\varepsilon-\frac{1}{m})^2}{(B_\tau-A_\tau)^2} \Big )m \Big \}.
\label{bound eq1}
\end{split}
\end{align}

Then we consider the second term
\begin{align}
\begin{split}
&  \sum_{j=1}^{m} P\Big\{  \sup _{\alpha \in \Lambda}  \big (  \int Q(z,\tau_j,\alpha)\dif F(\tau) \\
&\qquad \qquad\qquad\qquad- \frac{1}{l}\sum_{i=1}^{l} Q(z_i,\tau_j,\alpha) \big ) > \varepsilon \Big \}.
\end{split}
\end{align}
Similarly, we define the annealed entropy of data risk. Let $Q(z,\tau,\alpha),z\in Z,\alpha \in \Lambda$ be a set of real-valued functions. Let $N_z^{\Lambda,\beta_z}(z_1,...,z_l)$ be the number of different separations of $l$ vectors $z_1,...,z_l$ by a complete set of indicators:
\begin{align*}
&\theta \{ Q(z,\tau,\alpha) - \beta_z \},\\
&\ \alpha\in\Lambda,\ \beta_z \in \mathcal{B}_z=\Big ( \inf_{\alpha,t}Q(z,\tau,\alpha) \leq \beta_z \leq  \sup_{\alpha,t}Q(z,\tau,\alpha)  \Big ).
\end{align*}
The annealed entropy of data risk is defined that

\begin{definition}
	\textbf{Annealed Entropy of Sample Risk} 
	Let the function
	\begin{align}
	H_z^{\Lambda,\beta_z}(z_1,...,z_l) = \ln N_z^{\Lambda,\beta_z}(z_1,...,z_l)
	\end{align}
	be measurable with respect to measure on $z_1,...,z_l$.
	The quantity
	\begin{align}
	\hat{H}_{z}^{\Lambda,\beta_z}(l) = \ln E N_z^{\Lambda,\beta_z}(z_1,...,z_l)
	\end{align}
	is defined as the annealed entropy of the set indicators $\theta \{ Q(z,\tau,\alpha) - \beta_z \}$ of real-valued functions $Q(z,\tau,\alpha)$.
\end{definition}
And we have the inequality that
\begin{align}
\begin{split}
&\sum_{j=1}^{m} P\Big\{  \sup _{\alpha \in \Lambda}  \big (  \int Q(z,t_j,\alpha)\dif F(t) \\
&\qquad \qquad\qquad\qquad- \frac{1}{l}\sum_{i=1}^{l} Q(z_i,t_j,\alpha) \big ) > \varepsilon \Big \}
\end{split}\\
\begin{split}
\leq&\sum_{j=1}^{m} 4exp\Big \{ \Big ( \frac{\hat{H}_{z}^{\Lambda,\beta_z}(2l)}{l} - \frac{(\varepsilon-\frac{1}{l})^2}{(B_z-A_z)^2} \Big )l \Big \}
\end{split}\\
\begin{split}
=& 4exp\Big \{ \Big ( \frac{\ln m}{l} + \frac{\hat{H}_{z}^{\Lambda,\beta_z}(2l)}{l} - \frac{(\varepsilon-\frac{1}{l})^2}{(B_z-A_z)^2} \Big )l \Big \}
\end{split}
\label{bound eq2}
\end{align}

Let us substitute the inequation (\ref{bound eq1}) and (\ref{bound eq2}) into (\ref{pro eq1}), we get:
\begin{theorem}
	Let $A_t\leq R^{lo}(t,\alpha)\leq B_t, \alpha \in \Lambda$ and $A_z \leq Q(z,t,\alpha) \leq B_z,\alpha \in \Lambda$ be measurable set of bounded real-valued functions. Let $\hat{H}_{t}^{\Lambda,\beta_t}(m)$ and $\hat{H}_{z}^{\Lambda,\beta_z}(l)$ be the annealed entropies of the sets of indicators for them. Then the following inequality is valid:
	\begin{flalign}
	\begin{split}
	&P\Big \{  \sup _{\alpha \in \Lambda}  \big (  R(\alpha)-R_{emp} (\alpha, m, l) \big ) > \varepsilon \Big \} \\
	\leq& 4exp\Big \{ \Big ( \frac{\hat{H}_{t}^{\Lambda,\beta_t}(2m)}{m} - \frac{(\varepsilon-\frac{1}{m})^2}{(B_t-A_t)^2} \Big )m \Big \} \\
	&\  + 4exp\Big \{ \Big ( \frac{\ln m}{l} + \frac{\hat{H}_{z}^{\Lambda,\beta_z}(2l)}{l} - \frac{(\varepsilon-\frac{1}{l})^2}{(B_z-A_z)^2} \Big )l \Big \}
	\end{split} \label{th of cons eq}
	\end{flalign}
	\label{th of cons}
\end{theorem}

From this theorem, we can directly establish a sufficient condition for the uniform convergence, which is to satisfy three equations:
\begin{flalign}
& \lim_{l \rightarrow \infty} \frac{\hat{H}_{z}^{\Lambda,\beta_z}(l)}{l} = 0 \label{co eq 1}\\
& \lim_{m \rightarrow \infty} \frac{\hat{H}_{t}^{\Lambda,\beta_t}(m)}{m} = 0 \label{co eq 2}\\
& \lim_{l,m \rightarrow \infty} \frac{\ln m}{l} = 0 \label{co eq 3}
\end{flalign}
Note that in the Theorem \ref{lemma 1}, we have set the number of samples satisfied the inequality (\ref{sample relation}), which makes the equation (\ref{co eq 3}) always be true. So we replace the equation (\ref{co eq 3}) in the condition with inequality (\ref{sample relation}). And now we have the sufficient conditions of consistency that
\begin{corollary}
	For the existence of nontrival exponential bounds on uniform convergence, the sufficient conditions is to satisfy the following three formulas: 
	\begin{flalign}
	& \lim_{l \rightarrow \infty} \frac{\hat{H}_{z}^{\Lambda,\beta_z}(l)}{l} = 0 \\
	& \lim_{m \rightarrow \infty} \frac{\hat{H}_{t}^{\Lambda,\beta_t}(m)}{m} = 0 \\
	& l > \frac{2(B_z - A_z)^2}{\varepsilon ^2} \ln m + \frac{(B_z - A_z)^2}{(B_\tau - A_\tau)^2} m
	\end{flalign}
\end{corollary}
This condition is sufficient for the consistency of EGRM, but is not necessary. More discussion is needed for the necessary conditions, which will be demonstrated in the following paper.

\section{Triple Variables for Global Risk Controlling}
To analyze the error bound (\ref{th of cons eq}), we introduce the data dimension $h_z$ and subject dimension $h_tau$. These two variables is similar to the VC dimension in statistical learning theory. The data dimension $h_z$ corresponds to the function $Q(z,\tau,\alpha), \alpha \in \Lambda$. It is equal to the largest number of vectors $z_1,...,z_l$ that can be shattered by the complete set of indicators. Let the growth function of real-valued function $Q(z,\tau,\alpha)$ be
\begin{align}
G^{\Lambda, \mathcal{B}_z}(l) = \ln \max _{z_1,...,z_l} N^{\Lambda, \mathcal{B}_z} (z_1,...,z_l),
\end{align}
and we have
\begin{align}
\hat{H}_{z}^{\Lambda,\beta_z}(l) \leq G^{\Lambda, \mathcal{B}_z}(l) \leq h_z\Big ( \ln \frac{l}{h_z} + 1 \Big ).
\label{dim1}
\end{align}
The subject dimension corresponds to the function $R^{lo}(\tau, \alpha),\alpha \in \Lambda$. It is equal to the largest number of vectors $\tau_1,...,\tau_m$ that can be shattered by the complete set of indictors. Let the growth function of the real-valued function $R^{lo}(\tau, \alpha)$ be
\begin{align}
G^{\Lambda, \mathcal{B}_\tau}(m) = \ln \max _{\tau_1,...,\tau_m} N^{\Lambda, \mathcal{B}_\tau} (\tau_1,...,\tau_m),
\end{align}
and we have
\begin{align}
\hat{H}_{z}^{\Lambda,\beta_z}(l) \leq G^{\Lambda, \mathcal{B}_z}(l) \leq h_z\Big ( \ln \frac{l}{h_z} + 1 \Big ).
\label{dim2}
\end{align}
Directly take the inequalities (\ref{dim1}) and (\ref{dim2}) to the equation (\ref{th of cons eq}), we get the following probability
\begin{flalign}
\begin{split}
&P\Big \{  \sup _{\alpha \in \Lambda}  \big (  R(\alpha)-R_{emp} (\alpha, m, l) \big ) > \varepsilon \Big \} \\
\leq& 4exp\Big \{ \Big ( \frac{h_t}{m}(1+\ln \frac{2m}{h_\tau}) - \frac{(\varepsilon_{l,m}-\frac{1}{m})^2}{(B_\tau-A_\tau)^2} \Big )m \Big \} \\
&\  + 4exp\Big \{ \Big ( \frac{\ln m}{l} + \frac{h_z}{l}(1+\ln \frac{2l}{h_z}) - \frac{(\varepsilon_{l,m}-\frac{1}{l})^2}{(B_z-A_z)^2} \Big )l \Big \}.
\end{split} \label{th of cons eq2}
\end{flalign}
Now we write the above equation. Let the right side of the inequality (\ref{th of cons eq2}) be $\eta$, that is
\begin{align*}
\eta &= 4exp\Big \{ \Big ( \frac{h_t}{m}(1+\ln \frac{2m}{h_\tau}) - \frac{(\varepsilon_{l,m}-\frac{1}{m})^2}{(B_\tau-A_\tau)^2} \Big )m \Big \} \\
&\  + 4exp\Big \{ \Big ( \frac{\ln m}{l} + \frac{h_z}{l}(1+\ln \frac{2l}{h_z}) - \frac{(\varepsilon_{l,m}-\frac{1}{l})^2}{(B_z-A_z)^2} \Big )l \Big \}.
\end{align*}
Then we get the theorem
\begin{theorem}
	With probability $1-\eta$ the risk for the function $Q(z,t,\alpha_{l,m})$ which minimizes the empirical glob risk functional satisfies the inequality
	\begin{align}
	R(\alpha_{l,m}) < R_{emp}(\alpha_{l,m}) + \varepsilon_{l,m}.
	\label{error bound}
	\end{align}
	\vspace{-0.4cm}
\end{theorem}

This is the bound of generalization error in the main context.

\end{document}